\documentclass[11pt]{article}

\usepackage[final]{acl}

\usepackage{times}
\usepackage{latexsym}

\usepackage[T1]{fontenc}

\usepackage[utf8]{inputenc}

\usepackage{microtype}

\usepackage{inconsolata}

\usepackage{graphicx}
\usepackage{amssymb}
\usepackage{amsmath}
\usepackage{enumitem}
\usepackage{algorithm}
\usepackage{algpseudocode}
\usepackage{comment}
\usepackage{lipsum}
\usepackage{amsthm}
\usepackage{subcaption}
\usepackage{booktabs}

\newtheorem{proposition}{Proposition}

\title{MARS-RA: Rank Aggregation for Credit Assignment via Multimodal Comparisons in Embodied Multi-Agent Cooperation}

\author{
 \textbf{Dawei Wang\textsuperscript{1}},
  \textbf{Di Zhao\textsuperscript{2}},
  \textbf{Xinyuan Liu\textsuperscript{1}},
 \textbf{Marci Chi Ma\textsuperscript{1}},
 \\
  \textbf{Xiaoyang Liu\textsuperscript{1}},
 \textbf{Chengming Zhou\textsuperscript{1}},
  \textbf{Gary Ushaw\textsuperscript{1}},
 \textbf{Richard Davison\textsuperscript{1}},
\\
  \textsuperscript{1}Newcastle University, United Kingdom
  \\
 \textsuperscript{2}University of Auckland, New Zealand
}

\begin{document}
\maketitle
\begin{abstract}
Credit assignment is a fundamental challenge in cooperative multi-agent reinforcement learning, particularly in embodied AI settings characterized by limited and delayed feedback as well as dynamically changing numbers of active agents. We propose MARS-RA, a framework that reformulates credit assignment as a rank aggregation problem using contribution-based pairwise comparisons among agents generated by large multimodal models. This shift from absolute to relative estimation ensures robustness against noise and dynamic agent participation, converting comparison results into contribution scores for potential-based reward shaping. We provide theoretical justification for the convergence and robustness of the proposed framework, and show that Shapley values can be used as an interpretive reference. Experimental results on challenging tasks of different types indicate that MARS-RA can guide agents toward effective cooperation.

\end{abstract}

\section{Introduction}
\label{sec:intro}

\begin{figure*}
    \centering
    \includegraphics[width=1.0\textwidth]{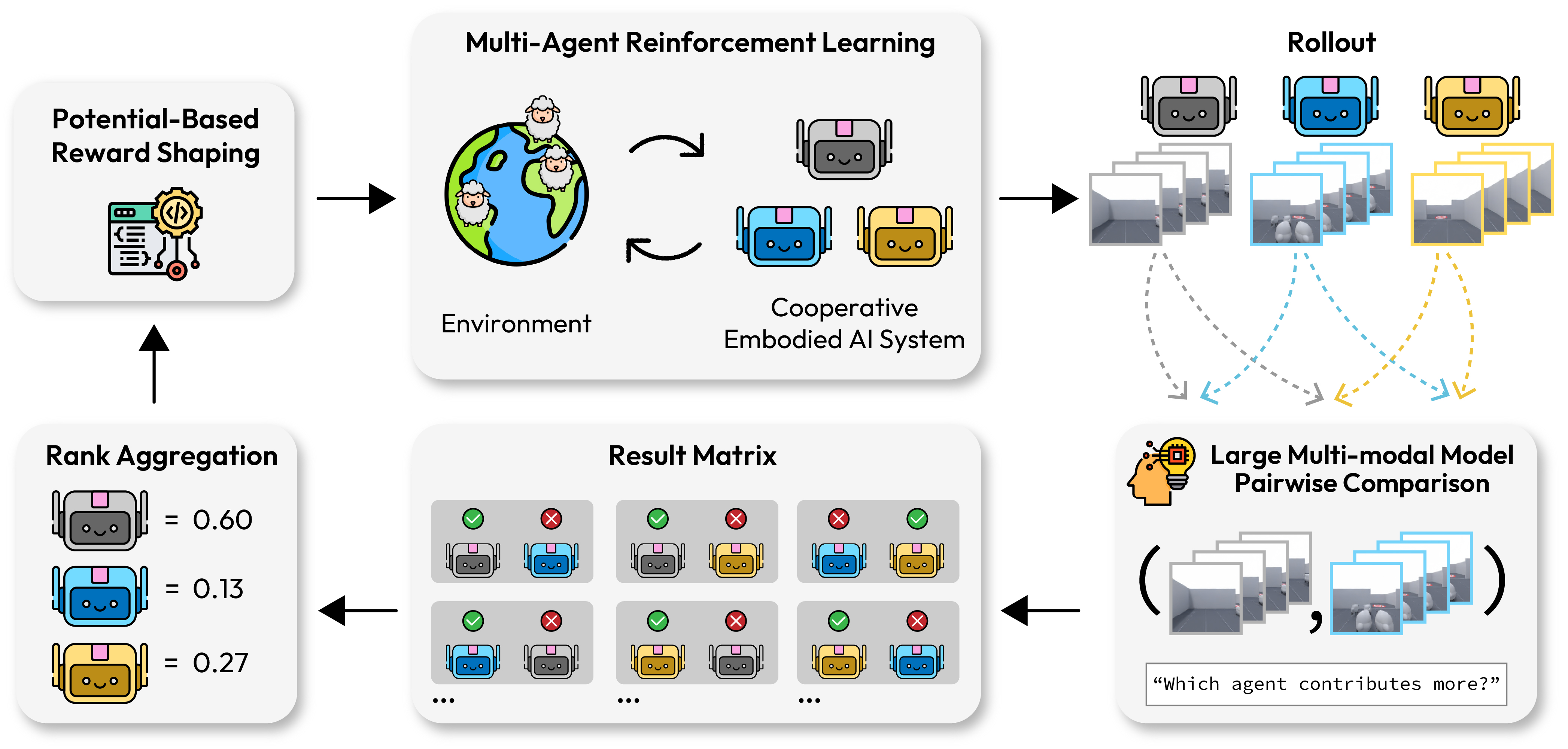}
    \caption{The MARS-RA framework for MARL credit assignment in cooperative embodied AI systems. Agents’ visual observations are processed by an LMM to conduct pairwise comparisons based on team contribution; the results are consolidated via rank aggregation to obtain contribution-aligned relevance scores, which are incorporated into MARL training through potential-based reward shaping.}
    \label{fig:fig1}
\end{figure*}

In Embodied Artificial Intelligence (AI) \cite{turing2021computing, clark1998being}, multi-agent reinforcement learning (MARL) has emerged as the canonical solution for enabling multiple agents to cooperate in dynamic environments \cite{zhang2021multi}. Credit assignment is a fundamental challenge in cooperative MARL, where an agent struggles to disentangle its own contribution to the global reward signal from the simultaneous actions of other agents in the environment \cite{minsky2007steps}. The credit assignment problem can lead to suboptimal policies and unpredictable behaviors \cite{wong2023deep}, which in turn result in concrete negative outcomes for cooperative embodied agents, such as wasted resources \cite{patel2023dream} or increased collision risks \cite{serra2023learning}. Therefore, addressing the credit assignment problem is a prerequisite for ensuring the safe and effective operation of cooperative embodied AI system.

However, the credit assignment problem in embodied AI is fundamentally exacerbated by two inherent characteristics: (1) Embodied agents perceive the world through partial and noisy multimodal egocentric sensors (e.g., RGB, thermal, LiDAR) \cite{feng2025embodied}, and they often operate in tasks with sparse rewards and long-horizon dependencies, resulting in limited and delayed feedback; (2) Embodied AI systems are open, where agents may leave or enter mid-task due to hardware failures or task demands, violating the standard fixed-agent-set assumption \cite{tang2023roma, abadi2025challenges}. Existing credit assignment methods, such as VDN \cite{sunehag2017value}, QMIX \cite{rashid2020monotonic}, and COMA \cite{foerster2017counterfactual}, already suffer performance degradation under insufficient feedback conditions alone \cite{papoudakis2020benchmarking}. Consequently, the credit assignment problem in cooperative embodied AI systems remains an unresolved challenge.

To address this challenge, we propose a novel approach that reformulates the credit assignment problem through the lens of rank aggregation \cite{debreu1960individual}, a probabilistic framework for inferring latent scores of $n$ items from multiple partially ordered samples (e.g., pairwise comparisons) \cite{ma2022tale}. This reformulation stems from viewing an agent’s credit as its latent contribution score, allowing us to estimate credits via rank aggregation over pairwise comparisons of \textit{“which agent contributes more”}. In complex multi-agent embodied cooperation tasks, precisely quantifying an individual agent's contribution is not only mathematically ill-posed \cite{bakushinsky2012ill} but also practically tenuous. The key advantage of this formulation is transforming absolute score estimation into relative pairwise comparisons, which makes credit assignment more tractable \cite{christiano2017deep}. Furthermore, this formulation intrinsically aligns with the inherent characteristics of embodied AI: (1) pairwise comparisons naturally accommodate dynamically changing numbers of active agents; (2) the inherent robustness of rank aggregation allows it to tolerate comparison noise arising either from biases in the pairwise-comparison generator itself or from embodied AI’s partial observability and long-horizon dependencies; and (3) the per-agent contribution scores produced by rank aggregation can serve as dense rewards that complement the environment’s sparse reward signals.



We propose \textbf{MARS-RA} (\textbf{M}ulti-\textbf{A}gent \textbf{R}eward \textbf{S}ystems via \textbf{R}ank \textbf{A}ggregation), a credit assignment framework that performs pairwise comparisons of agents’ contributions toward the team objective and employs rank aggregation to derive contribution-aligned relevance scores for each agent. These scores are then transformed into a potential function \cite{ng1999policy}, which is integrated into the MARL training loop to provide dense reward signals for learning. We utilize Large Multimodal Models (LMMs) \cite{team2023gemini, yang2025qwen3, liu2023visual, dong2025interleaved, zhao2026unlearning} as the automatic generator for the required pairwise comparisons. This automated pipeline is capable of meeting the computational demands of MARL training by eliminating the need for costly human annotation. LMMs possess advanced multimodal perception \cite{huang2023language, li2024lmeye, li2026videothinker, li2025videopro}, are capable of directly ingesting high-dimensional visual observations to perform spatio-temporal inference \cite{wang2026towermind, rocamonde2023vision, zhao2024symmetric}, and excel at making pairwise comparisons \cite{shi2024judging, di2025balancing}. The overall architecture of the MARS-RA is shown in Figure~\ref{fig:fig1}. Since existing embodied AI benchmarks lack settings where the number of active agents can change dynamically, we instantiate a challenging embodied multi-agent task suite in ManiSkill3 \cite{tao2024maniskill3} and construct \textbf{MARS-Bench}, where 2 to 4 decentralized embodied agents collaborate to complete three representative tasks: \textit{Pass Gate}, \textit{Herd Sheep}, and \textit{Collect Ball}. Agents in these tasks operate under partial egocentric observations and sparse rewards, and may enter or exit during task execution, closely mirroring the characteristics of embodied AI scenarios.

The main contributions of this work are as follows: (1) We pioneer a reformulation of the credit assignment problem in MARL as a rank aggregation problem. This formulation is well-suited to embodied AI settings and naturally integrates with LMM-based automated comparisons and potential-based reward shaping. (2) We propose MARS-RA, which outperforms strong baselines in our experiments. Further analysis shows that its performance improves with higher accuracy and increased number of LMM-based pairwise comparisons, suggesting that MARS-RA can benefit from continued advances in LMM reasoning capability. (3) We build MARS-Bench, a benchmark that captures key characteristics and challenges of embodied AI settings, particularly by introducing openness in the number of active agents, encouraging the community to move beyond static team assumptions and develop more resilient algorithms.


\section{Related Work}
\label{sec:related_work}

\noindent \textbf{Multi-Agent Cooperation in Embodied AI.} Embodied AI \cite{brooks1991new} refers to intelligent agents equipped with physical bodies or virtual embodiments that can perceive, act, and adapt through continuous interaction with their environment. Embodied AI tasks often involve cooperation among multiple agents, which can exhibit capabilities beyond those of individual agents. Existing benchmarks for multi-agent embodied AI, such as MQE \cite{xiong2024mqe} and TDW-Cook \cite{zhang2024combo}, typically exhibit characteristic challenges like partial observability and sparse rewards. However, none of the currently available benchmarks support a dynamic number of active agents, representing a notable gap in the field. 





\noindent \textbf{Credit Assignment.} Credit assignment is a fundamental challenge in MARL. To address this challenge, various approaches have been proposed to improve credit assignment: value-decomposition methods such as VDN \cite{sunehag2017value} and QMIX \cite{rashid2020monotonic} factorize a centralized Q-function into individual agent Q-functions; COMA \cite{foerster2017counterfactual} applies a counterfactual advantage baseline to estimate the individual contributions of each agent; SQDDPG \cite{wang2020shapley} utilizes Shapley values as a principled mechanism for credit assignment. Large language models (LLMs) have opened new directions for credit assignment through language-based reasoning and coordination, with works such as LLM-MCA \cite{nagpal2025leveraging}, LCA \cite{lin2025speaking}, SAMA \cite{li2025multi}, and LERO \cite{wei2025lero} leveraging LLMs for credit evaluation, task decomposition, and hybrid reward design. However, these methods often rely on simplified environments, handcrafted rules, and strong assumptions, limiting their applicability to cooperative embodied AI scenarios.

\noindent \textbf{Rank Aggregation.} Rank aggregation is an important task across a wide range of disciplines, including sports \cite{herbrich2006trueskill}, psychology \cite{critchlow1991probability}, and bioinformatics \cite{kolde2012robust}. In essence, rank aggregation methods treat pairwise comparisons as a means of estimating the latent ‘quality’ or ‘score’ of the compared items, such as the popularity of books or the skill levels of athletes. In the machine learning domain, rank aggregation has been applied in various areas. In this work, we pioneer the introduction of rank aggregation in MARL, reinterpreting the credit assignment problem through the lens of rank aggregation.




\section{Preliminaries}
\label{sec:preliminaries}
Embodied multi-agent cooperation can be modeled within the framework of an open decentralized partially observable Markov decision process (Open Dec-POMDP) \cite{cohen2017open}, defined by $( \mathcal{N}, \mathcal{I}, \mathcal{S}, \mathcal{A}, \mathcal{O}, \phi, r, \gamma, T )$, where: $\mathcal{N} = \{1, 2, \ldots, n\}$ is a finite population of $n$ agents. $\mathcal{I} \subseteq \mathcal{P}(\mathcal{N})$ is a finite set of coalitions formed from the agent population $\mathcal{N}$. The coalition $I^t \in \mathcal{I}$ at time $t$ is designated as the \textit{operating coalition}, and an agent $i$ is defined as \textit{active} at time $t$ if $i \in I^t$. $\mathcal{S}$ is the finite set of states. $\mathcal{A} = \{\mathcal{A}^i \mid i \in N\}$ represents the set of action spaces, where $\mathcal{A}^i$ is the finite set of actions for agent $i$, $\mathcal{A}^I = \times_{i \in I} \mathcal{A}^i$ denotes the set of joint actions $a^I$ available to coalition $I$, $\boldsymbol{a}^I = (a^i)_{i \in I}$. $\mathcal{O} = \{\mathcal{O}^i \mid i \in \mathcal{N}\}$ represents the set of observation spaces, where $\mathcal{O}^i$ is the finite set of observations for agent $i$, $\mathcal{O}^I = \times_{i \in I} \mathcal{O}^i$ denotes the set of joint observations $\boldsymbol{o}^I$ available to coalition $I$, $\boldsymbol{o}^I = (o^i)_{i \in I}$. $\phi$ is the dynamics model, where $\phi(s', \boldsymbol{o}^{I'}, I' \mid s, I, \boldsymbol{a}^I) \in [0, 1]$ specifies the probability of transitioning to state $s'$ and coalition configuration $I'$, while receiving joint observation $\boldsymbol{o}^{I'}$, given that coalition $I$ took joint action $\boldsymbol{a}^I$ in state $s$. $r: \mathcal{S} \times \mathcal{A}^I \to \mathbb{R}$ is the reward function, defining the reward $r(s, \boldsymbol{a}^I)$ received after coalition $I$ executes joint action $\boldsymbol{a}^I$ in state $s$. $\gamma \in [0, 1]$ is the discount factor. $T$ is the planning horizon. Our objective is to learn the joint policy $\pi$ that maximizes the expected discounted sum of rewards, where the rewards $(R_t)_{t \in {0, 1, \dots, T-1}}$ are random variables distributed according to the reward function $r$: $\underset{\pi}{\operatorname{arg\,max}} \; \mathbb{E} \{ \sum_{t=0}^{T-1} \gamma^t R_t \mid \pi \}$.

\begin{figure}
    \centering
    \includegraphics[width=\columnwidth]{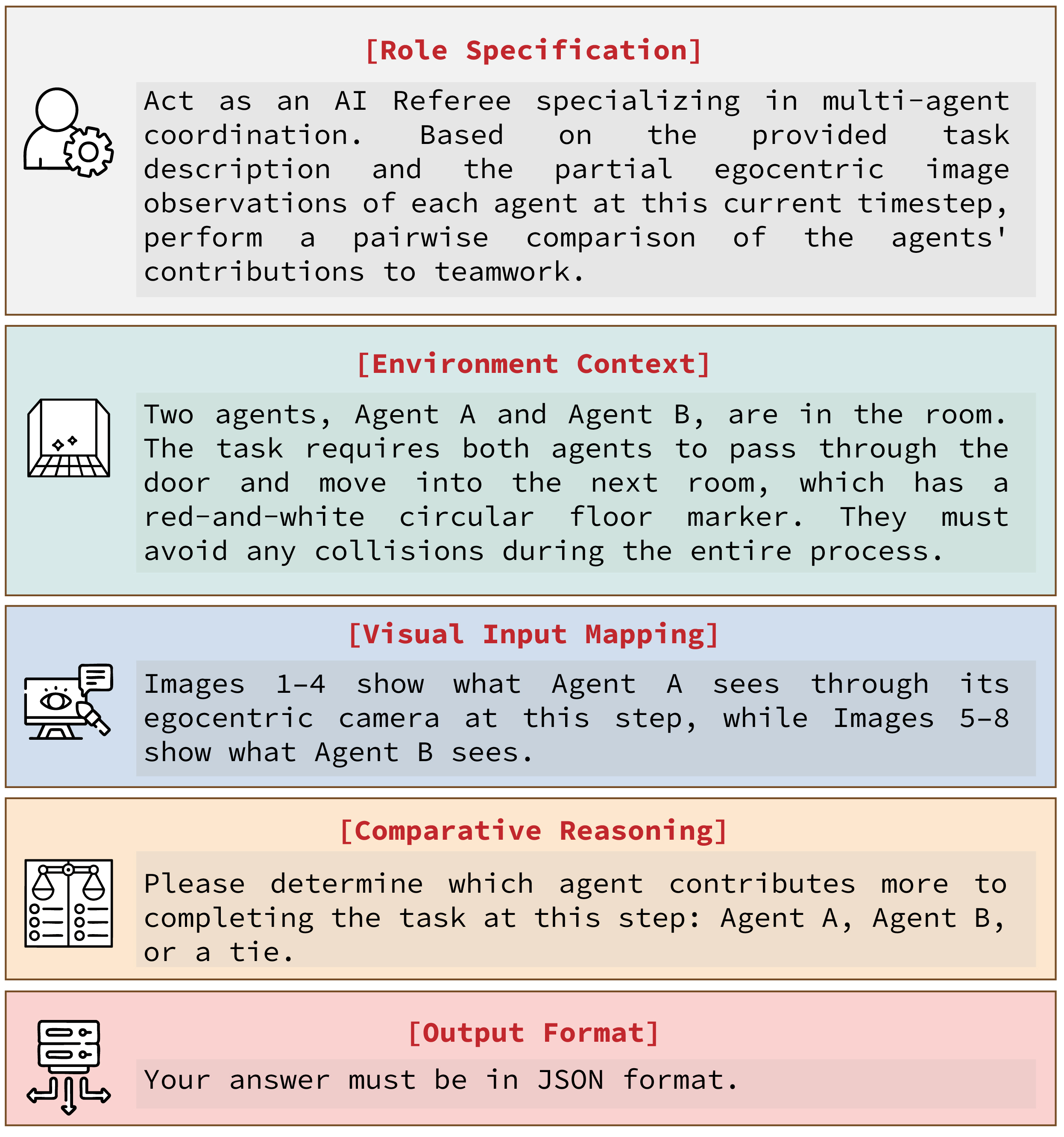}
    \caption{Example of the prompt used in the Pass Gate task. The prompt consists of five components: role specification, environmental context, visual input mapping, comparative reasoning, and output format.}
    \label{fig:prompt}
\end{figure}

\section{Assumptions}
\label{sec:assumptions}
The following assumptions are imposed on the LMMs employed in this work: (1) We assume that the LMMs are trained on diverse text and image corpora, providing a reasonable basis for generalization across embodied AI scenarios and tasks. (2) We assume that LMMs can process multiple images concurrently and follow textual prompts to perform reasoning. (3) MARS-RA is intended for tasks whose states can be evaluated using image-based observations and textual information.

\section{Method}
\label{sec:method}
As shown in Figure~\ref{fig:fig1}, MARS-RA is a framework that integrates rank aggregation and LMMs to enable automatic credit assignment in MARL training. It consists of three steps: (1) LMM-based pairwise comparison, simplifying the contribution evaluation task into relative judgments, thereby providing the essential supervisory signal required for subsequent rank aggregation; (2) rank aggregation, serving to transform the discrete comparison matrix into continuous contribution credits, providing a fine-grained numeric basis for downstream MARL; and (3) potential function, transforming static contribution credits into dynamic shaping rewards, seamlessly embedding the ranking outcomes as dense rewards into the training loop.

\subsection{LMM-based Pairwise Comparison.} 
Pairwise comparison is inherently independent of population size, making it an ideal mechanism for extracting reliable information in scenarios with a dynamic number of active agents. LMMs are especially well suited for performing these comparisons in MARS-RA: they can reason directly over high-dimensional visual observations that are difficult to evaluate using low-level states or hand-crafted rules, and they enable a fully automated assessment pipeline that naturally scales with compute-intensive MARL training.

During training, at each step $t$, we construct a pairwise comparison matrix $\boldsymbol{M}_t \in \mathbb{R}^{n \times n}$ initialized to zero, where $n$ represents the total number of agents. We then query the LMM to perform pairwise comparisons only among the active agents in the set $I^t$. For every ordered pair $(i,j)$ where $i \neq j$ and $i, j \in I^t$, the respective egocentric image observations $o_t^i$ and $o_t^j$ are provided as inputs to the LMM. Thus, each unordered agent pair corresponds to two ordered comparisons, which helps mitigate position bias \cite{tian2025identifying} in LMM-based judgments. The resulting preferences (win for $i$, win for $j$, or tie) are assigned to the corresponding entries in $\boldsymbol{M}_t$, while entries involving inactive agents remain zero:
\begin{equation}
    (\boldsymbol{M}_t)_{i,j} = f_{\text{LMM}}(o_t^i, o_t^j, \Theta),
\end{equation}

\noindent where $\Theta$ denotes the textual prompt instructing the model to judge which agent contributes more to the team. $\boldsymbol{M}_t$ is structured as a comparison outcome matrix (2d array) where each entry $(\boldsymbol{M}_t)_{i,j}$ represents the number of times agent $i$ is preferred over agent $j$. To handle ties, we assign a score of $0.5$ to both $(\boldsymbol{M}_t)_{i,j}$ and $(\boldsymbol{M}_t)_{j,i}$ for each neutral judgment. It is important to note that, since pairwise comparison requires at least two entities, when the number of active agents satisfies \( |I^{t}| < 2 \), we simply skip the query process.

Figure~\ref{fig:prompt} illustrates the structure of the prompt, which consists of five components: (1) role specification: assign the LMM the persona of a referee; (2) environmental context: supply the LMM with information about the task and its rules; (3) visual input mapping: indicate to the LMM the agent identity associated with each image; (4) comparative reasoning: instruct the LMM to determine which agent made a greater contribution to the team's success; (5) output format: enforce that the LMM outputs a format that is easy to parse (e.g., JSON).


\subsection{Rank Aggregation.} 
We adopt the Bradley\textendash Terry model \cite{bradley1952rank} to synthesize local pairwise comparisons into global contribution scores. This probabilistic framework serves two critical purposes: First, it performs Maximum Likelihood Estimation (MLE) \cite{bishop2006pattern} aggregation under a scalar latent-score model, which can mitigate noise or inconsistency by best-fit ranking. Second, it transforms ordinal preferences into continuous cardinal values, capturing not just the rank order but the magnitude of capability differences between agents. This fine-grained quantification is essential for generating smooth and informative potential-based rewards. 

We fit the Bradley\textendash Terry model to infer a latent contribution score 
\(
\boldsymbol{c}_t \in \mathbb{R}^{|I^t|}
\)
for all active agents at step \(t\). The probability that agent \(i\) is preferred over agent \(j\) is modeled as:
\begin{equation}
\mathbb{P}(i \succ j)
= \frac{\exp(c_t^{i})}{\exp(c_t^{i}) + \exp(c_t^{j})}.
\end{equation} 
Given $\boldsymbol{M}_t$, we estimate \(\boldsymbol{c}_t\) by minimizing the negative log-likelihood:
\begin{equation}
    \hat{\boldsymbol{c}_t} = \mathop{\arg\min}_{\boldsymbol{c}_t} \sum_{i, j} \left[- (\boldsymbol{M}_t)_{i,j} \log \mathbb{P}(i \succ j)\right]
\end{equation}
The resulting estimate $\hat{\boldsymbol{c}_t}$ serves as the credit assigned to each agent at step $t$.

\subsection{Potential Function.} 
To incorporate the derived contribution scores without biasing the optimization objective, we adopt potential-based reward shaping (PBRS) \cite{ng1999policy}. This is crucial because the contribution scores evolve during training. Prior work shows that PBRS preserves policy invariance and Nash equilibria even under dynamic shaping \cite{devlin2012dynamic}, and can safely embed arbitrary rewards by expressing them as dynamic potentials \cite{harutyunyan2015expressing}.


At this point, the reward function becomes $r: \mathcal{S} \times \mathcal{A}^I \times \mathcal{S} \to \mathbb{R}$, with \( r(s,     \boldsymbol{a}^I, s') \) denoting the reward received when coalition \( I \) executes the joint action \( \boldsymbol{a}^I \) in state \( s \) and transitions to the next state \( s' \). And the shaping reward function in this mechanism is defined as $F(s, t, s', t') = \gamma \, \psi(s', t') - \psi(s, t)$, where \( t \) denotes the time at which the agent was in the previous state \( s \), and \( t' \) is the time when it reaches the current state \( s' \), with \( t < t' \). We define our potential function as:
\begin{equation}
    \psi(s_t, t) = 
    \begin{cases} 
        0, & \text{if } s_t \text{ is terminal}, \\
        \operatorname{softmax}(\hat{\boldsymbol{c}_t}), & \text{otherwise}.
    \end{cases}
\end{equation}
This assigns a zero potential to the terminal state \cite{wierstra2008episodic}, and applies softmax normalization at all non-terminal steps to map the contribution scores into a normalized representation of relative strength. We employ the following shaping reward under this mechanism: $F(s_t, t, s_{t+1}, t+1) = \gamma  \psi(s_{t+1}, t+1) -  \psi(s_t, t)$. The final shaped reward used for training is:
\begin{equation}
\tilde{\boldsymbol{r}}_t = r(s_t, \boldsymbol{a}_t^I, s_{t+1}) + \rho F(s_t, t, s_{t+1}, t+1),
\end{equation}
where $\rho$ is a scalar weighting factor that modulates the contribution of the potential-based reward relative to the environmental reward.

\section{Theoretical Properties}
\label{sec:theory}


In this section, we provide the theoretical justification for MARS-RA. We analyze two key properties: (1) \textbf{Convergence and Robustness}, showing that our rank aggregation framework can recover agents’ underlying contribution scores and effectively mitigate noise and inaccuracies in LMM-generated pairwise comparisons, including those induced by LMM hallucinations as well as by partial observability and long-horizon dependencies in embodied AI scenarios; and (2) \textbf{Shapley Values as an Interpretive Reference}, demonstrating the consistency of our derived scores with Shapley values \cite{shapley1953value} under ideal conditions. We follow the notation established in the preceding two sections. The detailed proof is provided in the Appendix~\ref{appendix:proofs}.

\subsection{Convergence and Robustness Analysis}
\label{subsec:robustness}

We assume that the LMM possesses a latent ground-truth preference vector $\mathbf{c}_t^* \in \mathbb{R}^{|I^t|}$ that represents the relative contributions of the active agents at step $t$. The LMM does not output $\mathbf{c}_t^*$ directly; instead, it acts as a comparator that follows the Bradley–Terry model and uses Maximum Likelihood Estimation (MLE) to obtain an estimator $\hat{\mathbf{c}_t}$ that best explains the observed pairwise comparisons. We now characterize the error bound of this estimator.

\begin{proposition}[Convergence and Robustness]
\label{pro:robustness}
Suppose the comparison graph formed by $\boldsymbol{M}_t$ is connected. Let $\hat{\mathbf{c}_t}$ be the MLE estimate derived from $K$ pairwise comparisons. With probability at least $1 - n^{-2}$, the estimation error relative to the latent preference $\mathbf{c}_t^*$ satisfies:
\begin{equation}
    \frac{1}{\sqrt{n}} \|\hat{\mathbf{c}_t} - \mathbf{c}_t^*\|_2 \le C_0 \sqrt{\frac{n \log n}{K}}
\end{equation}
where $C_0$ is a constant depending on the graph topology and $\|\cdot\|_2$ denotes the Euclidean norm.
\end{proposition}

\begin{proof}[Sketch of Proof]
This result relies on the analysis of regularized MLE for pairwise comparisons \cite{negahban2012iterative}. The Hessian of the negative log-likelihood behaves similarly to the Laplacian of the comparison graph \cite{shah2016estimation}. Under the assumption of algebraic connectivity ($\lambda_2 > 0$), the objective function satisfies Restricted Strong Convexity \cite{negahban2012unified}. Combined with the concentration of measure for the gradient (bounded by $O(\sqrt{K \log n})$ via Hoeffding's inequality), standard convex optimization analysis yields the convergence rate of $O(1/\sqrt{K})$.
\end{proof}

Proposition \ref{pro:robustness} provides a mathematical justification for the convergence and robustness of our framework. It guarantees that as long as we perform sufficient pairwise comparisons ($K$), the aggregated scores $\hat{\mathbf{c}_t}$ will converge to the LMM's stable latent preference $\mathbf{c}_t^*$, thereby filtering out the influence of noise and inaccuracies.


\subsection{Shapley Values as an Interpretive Reference}
\label{subsec:consistency}

Having established that our method robustly recovers the LMM's latent preference $\mathbf{c}_t^*$, a natural question arises: \textit{What does $\mathbf{c}_t^*$ represent conceptually?} We posit that a rational LMM, when provided with sufficient context, judges agents based on their marginal contributions.

\begin{proposition}[Interpretability via Shapley Values]
If the LMM acts as a rational probabilistic comparator where the latent preference is determined by the agents' true Shapley values (i.e., $\mathbf{c}_t^* = \mathbf{v}_t^*$), then the scores $\hat{\mathbf{c}_t}$ derived by MARS-RA are consistent estimators of the true Shapley values (up to a translation constant).
\end{proposition}

This proposition serves as an interpretability bridge: it links the statistically robust scores derived in Section \ref{subsec:robustness} to the game-theoretic concept of contribution assignment. This result is not intended to suggest that MARS-RA recovers true Shapley values in practice, but rather to provide an interpretive lens for understanding the semantics of the aggregated scores under idealized assumptions.

\section{MARS-Bench}

\begin{figure}
    \centering
    \includegraphics[width=\columnwidth]{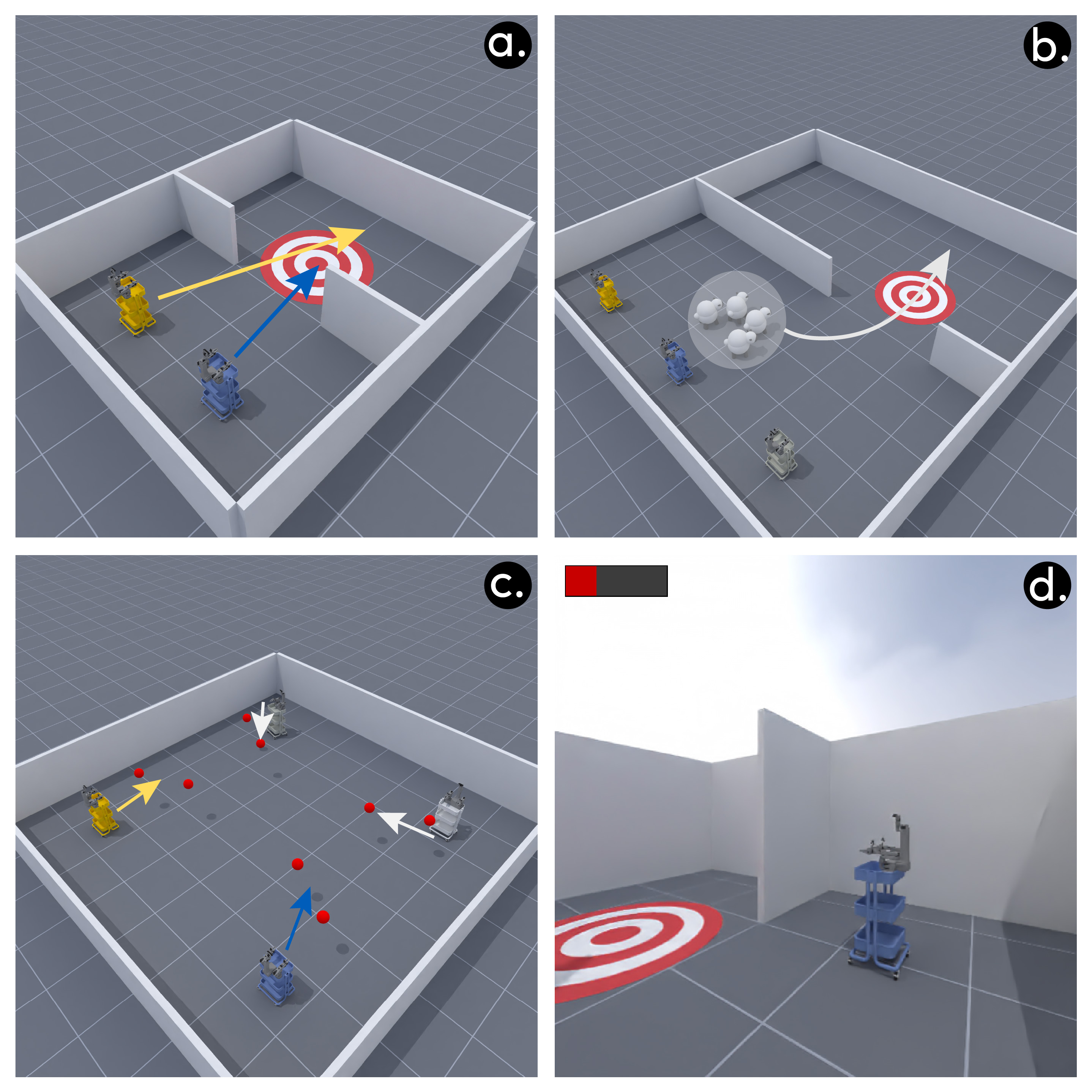}
    \caption{(a–c) Visualizations of the Pass Gate, Herd Sheep, and Collect Ball tasks in MARS-Bench, respectively. Colored arrows depict the intended movements of agents and objects. (d) An example of an agent’s egocentric observation. The icon in the upper-left corner denotes the agent’s current energy level; once depleted, the agent is removed from the environment and re-enters with full energy after a random number of steps.}
    \label{fig:mars-bench}
\end{figure}

\begin{figure*}[htbp]
\centering
\begin{subfigure}{0.32\linewidth}
  \centering
  \includegraphics[width=\linewidth]{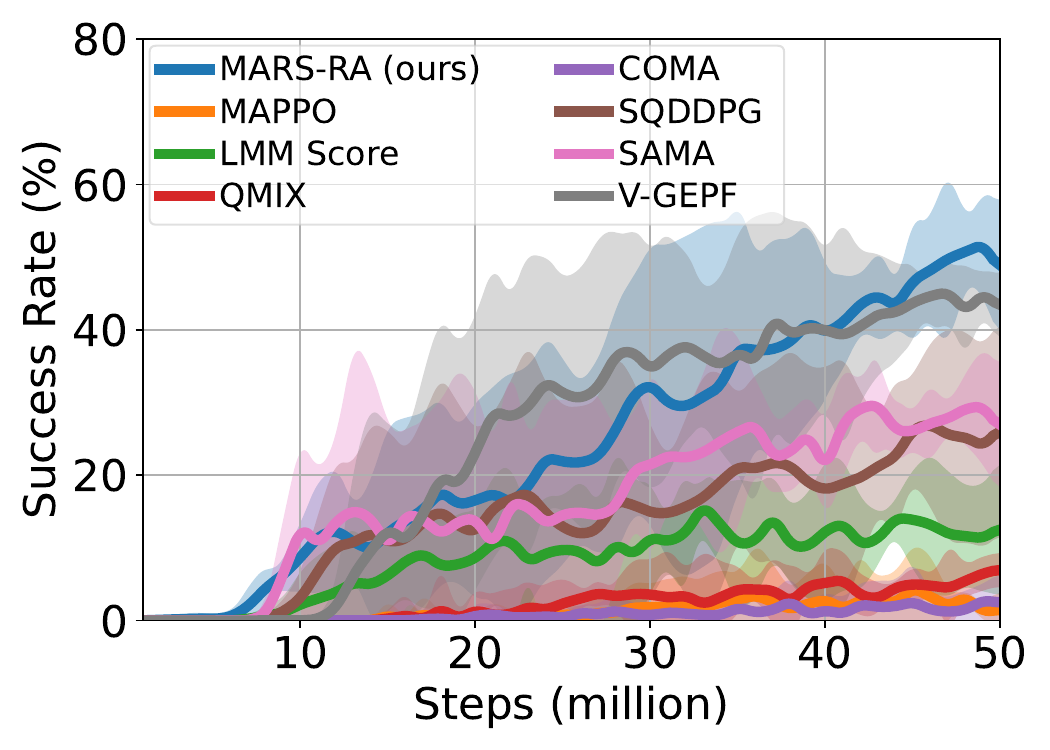}
  \caption{Pass Gate}
  \label{fig:A}
\end{subfigure}
\hfill
\begin{subfigure}{0.32\linewidth}
  \centering
  \includegraphics[width=\linewidth]{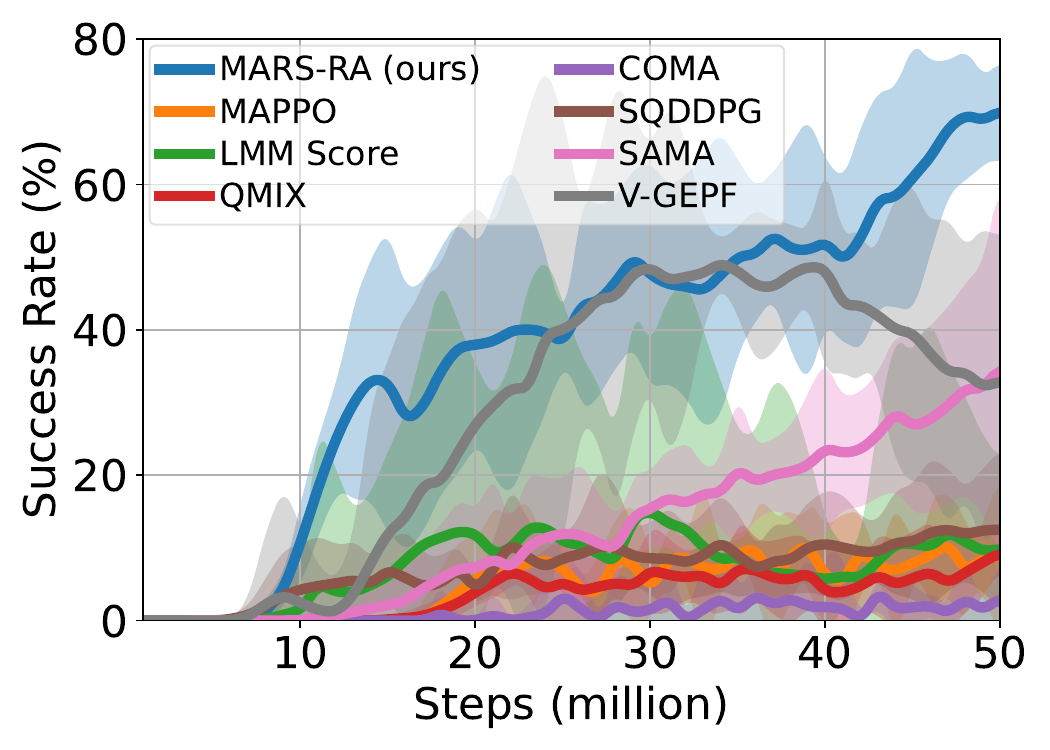}
  \caption{Herd Sheep}
  \label{fig:B}
\end{subfigure}
\hfill
\begin{subfigure}{0.32\linewidth}
  \centering
  \includegraphics[width=\linewidth]{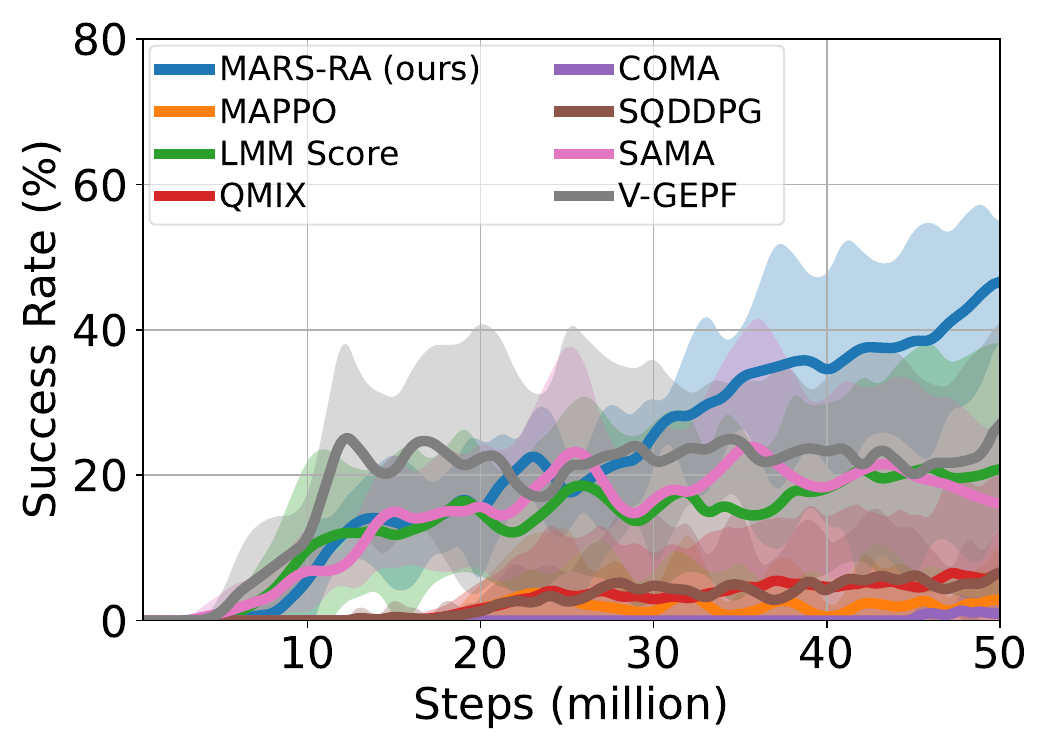}
  \caption{Collect Ball}
  \label{fig:C}
\end{subfigure}
\caption{Learning curves of all compared methods on the three tasks in MARS-RA, trained for 50 million environment steps. Results are averaged over five random seeds, and error bars indicate 95\% confidence intervals.}
\label{fig:overall}
\end{figure*}

Existing embodied AI benchmarks are limited by their fixed-agent settings, failing to account for scenarios with a dynamic number of active agents. We implement a challenging set of embodied cooperative tasks with 2 to 4 agents in ManiSkill3 using XLeRobot \cite{wang2025xlerobot}, and construct MARS-Bench, as illustrated in Figure~\ref{fig:mars-bench}. It uses agents’ egocentric camera views as pixel-based observations and adopts a discrete action space. Two reward modes are provided: sparse rewards for training to closely reflect real-world conditions, and dense rewards for analytical evaluation and debugging. At the start of each episode, agents are assigned random initial battery levels, which are depleted by actions. Agents with depleted batteries are temporarily removed and later respawn with full charge after a random delay, introducing agent-number openness. It consists of three tasks: \textbf{Pass Gate}, two agents are required to traverse a doorway between rooms without collisions; \textbf{Herd Sheep}, three agents cooperatively herd sheep from one room to another, where the sheep follow predefined movement dynamics; \textbf{Collect Ball}, four agents collect all red balls in the room. These tasks correspond to three fundamental multi-agent cooperation scenarios, namely spatio-temporal movement, cooperation, and divide and conquer \cite{wu2021too}. See Appendix~\ref{appendix:mars-bench} for more details.

\section{Experiment}
\label{sec:experiment}

We evaluate MARS-RA on MARS-Bench, using the task success rate as the primary metric. Specifically, success is defined by task completion, whereas collisions and timeouts are failures. Beyond MARS-Bench, we further evaluate the generalization of MARS-RA on Overcooked \cite{carroll2019utility} and Pistonball \cite{terry2021pettingzoo}, with detailed descriptions of both environments provided in Appendix~\ref{appendix:more_env}. MARS-RA adopts MAPPO \cite{yu2022surprising} as the backbone algorithm, while the required pairwise comparisons are generated by Gemini-2.5-Pro \cite{comanici2025gemini}, with a single query per comparison. For comparison, we include MAPPO and LMM Score as ablation variants alongside the baseline methods. See Appendix~\ref{appendix:exp_details} for experiment details.

\begin{itemize}[label=\textendash, noitemsep, topsep=0pt, leftmargin=*]
    \item \textbf{MAPPO}. It is a widely used policy-gradient algorithm for cooperative MARL and is adopted as the backbone of our method.
  \item \textbf{LMM Score}. This baseline queries the LMM using the task description and all agents’ observations, generating per-agent contribution scores in $[0,1]$ that are incorporated into learning via potential-based reward shaping.
  \item \textbf{QMIX} \cite{rashid2020monotonic}, \textbf{COMA} \cite{foerster2017counterfactual} and \textbf{SQDDPG} \cite{wang2020shapley}. These widely used methods address the credit assignment problem through value decomposition, counterfactual advantage estimation, and Shapley value approximation, respectively.
  \item \textbf{SAMA} \cite{li2025multi}. This is a subgoal-based framework designed to address the credit assignment problem in cooperative MARL. It leverages the commonsense priors embedded in LMMs to guide agent coordination, using MAPPO as its underlying backbone algorithm.
  \item \textbf{V-GEPF} \cite{ma2025vision}. It is a hierarchical reward-shaping framework built on MAPPO for cooperative MARL. It employs a potential function derived from a vision language model for semantic guidance, alongside a LMM that selects cooperative skills from a predefined pool.
\end{itemize}



\begin{figure*}[htbp]
\centering
\begin{subfigure}{0.32\linewidth}
  \centering
  \includegraphics[width=\linewidth]{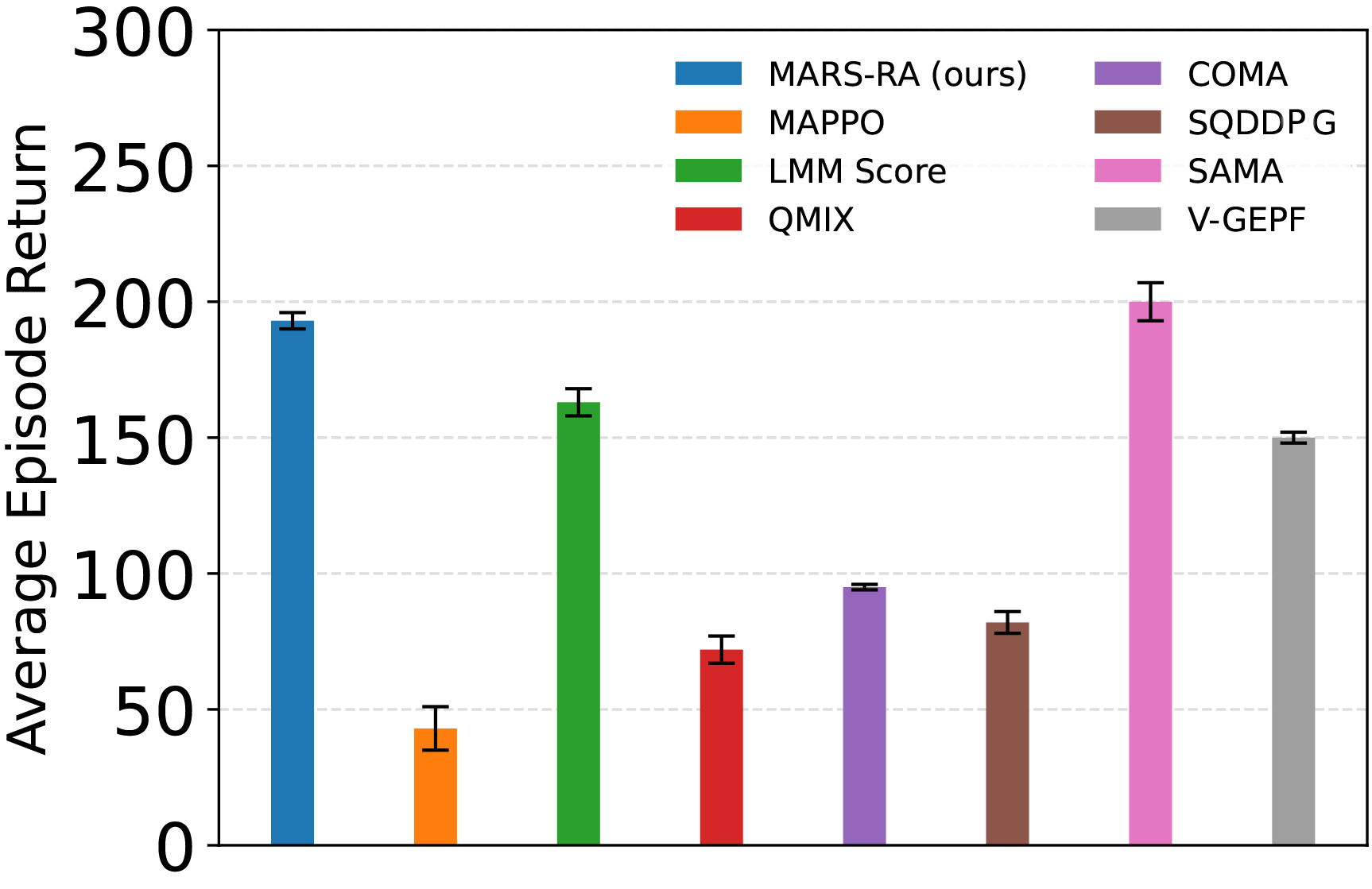}
  \caption{Cramped Room}
  \label{fig:bar1}
\end{subfigure}
\hfill
\begin{subfigure}{0.32\linewidth}
  \centering
  \includegraphics[width=\linewidth]{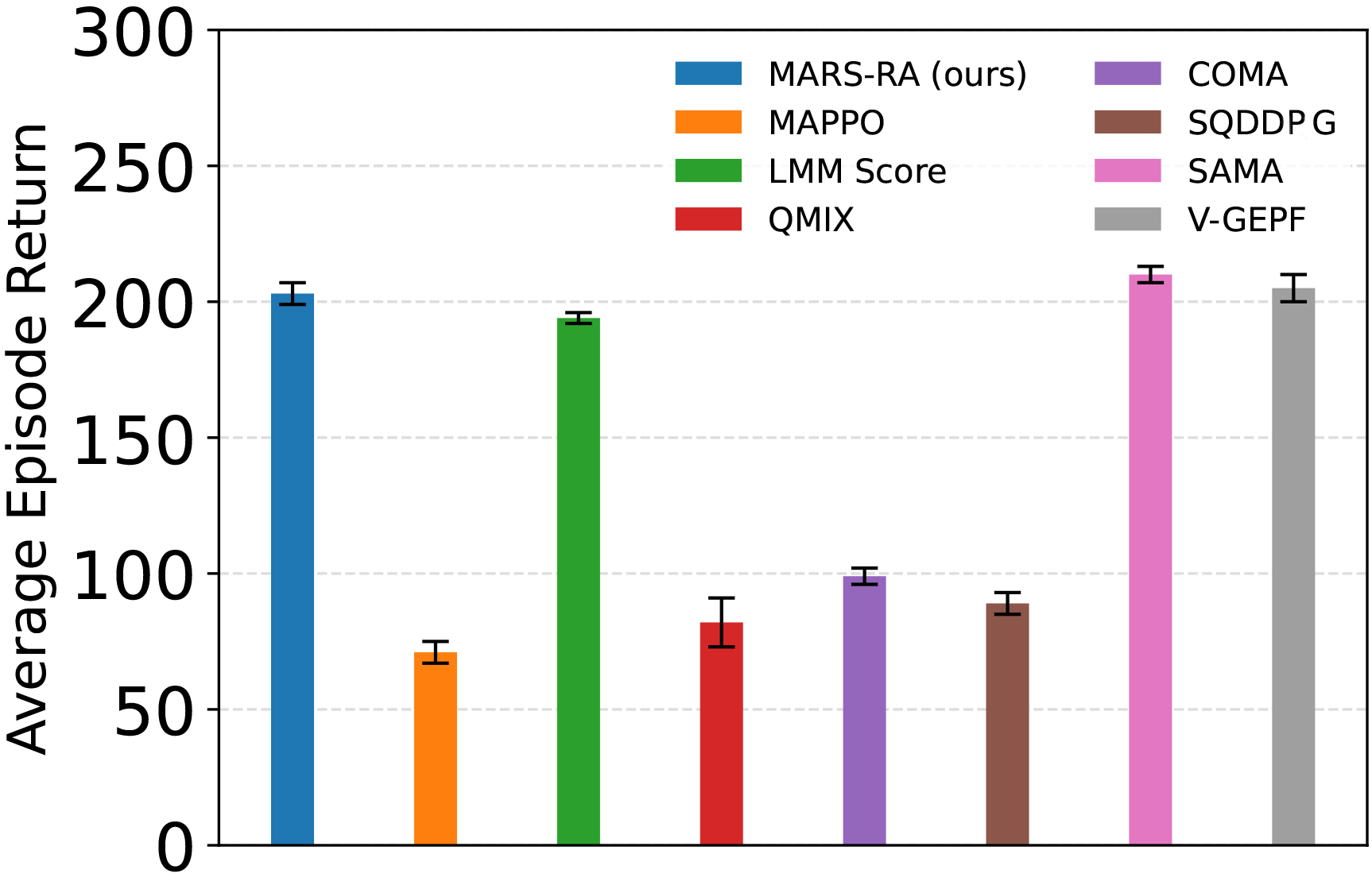}
  \caption{Asymmetric Advantages}
  \label{fig:bar2}
\end{subfigure}
\hfill
\begin{subfigure}{0.32\linewidth}
  \centering
  \includegraphics[width=\linewidth]{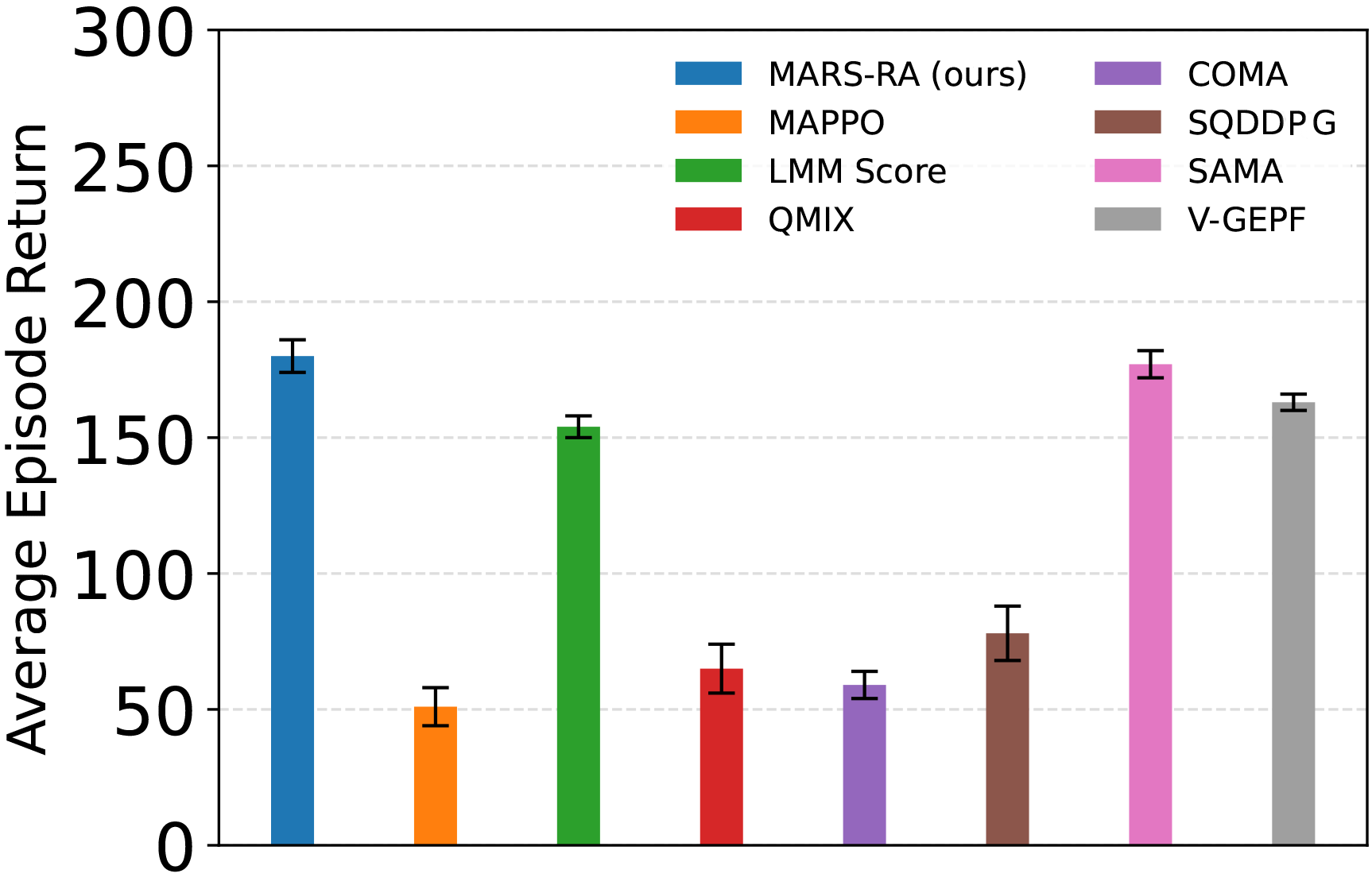}
  \caption{Coordination Ring}
  \label{fig:bar3}
\end{subfigure}
\hfill
\begin{subfigure}{0.32\linewidth}
  \centering
  \includegraphics[width=\linewidth]{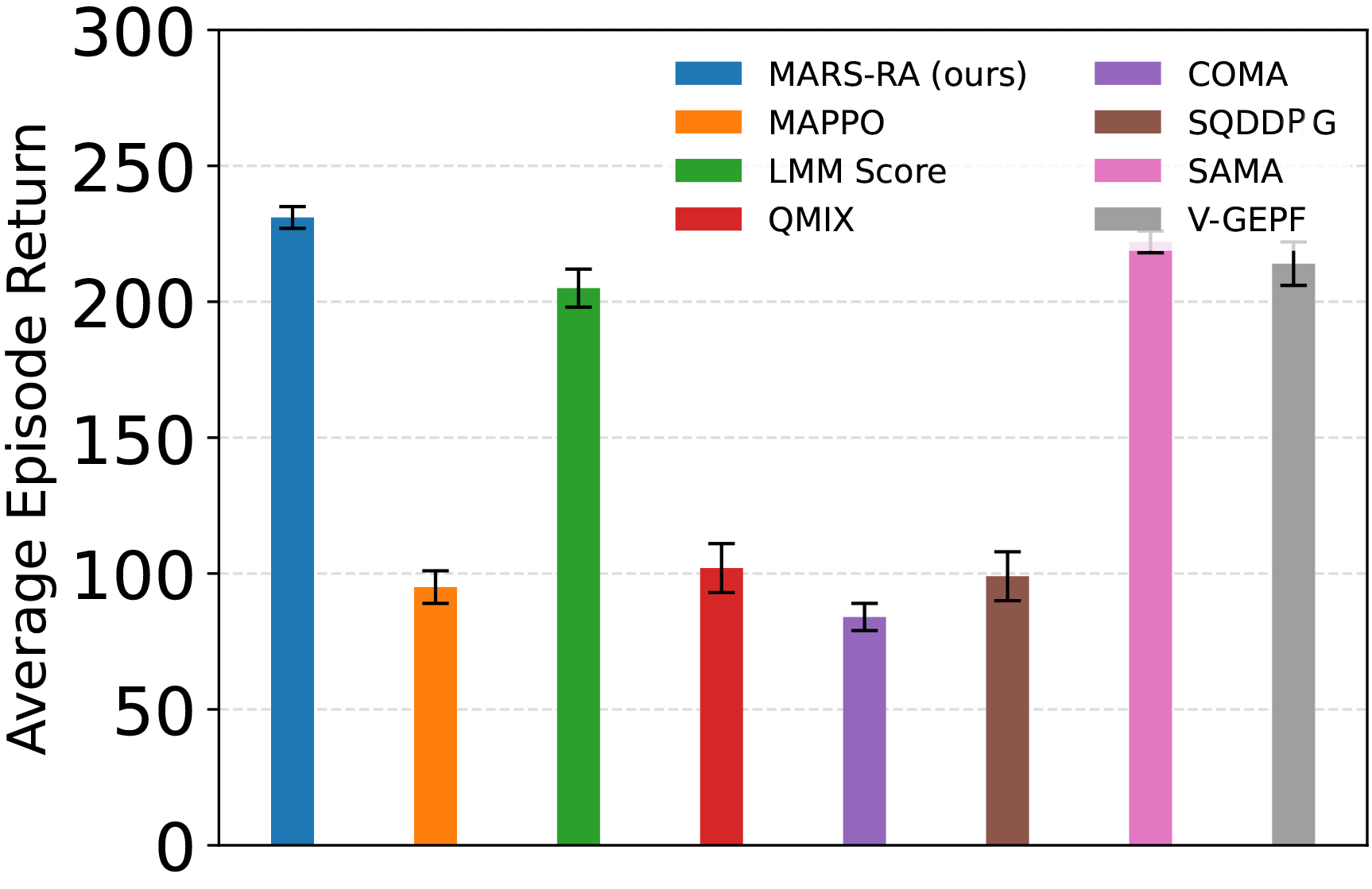}
  \caption{Forced Coordination}
  \label{fig:bar4}
\end{subfigure}
\hfill
\begin{subfigure}{0.32\linewidth}
  \centering
  \includegraphics[width=\linewidth]{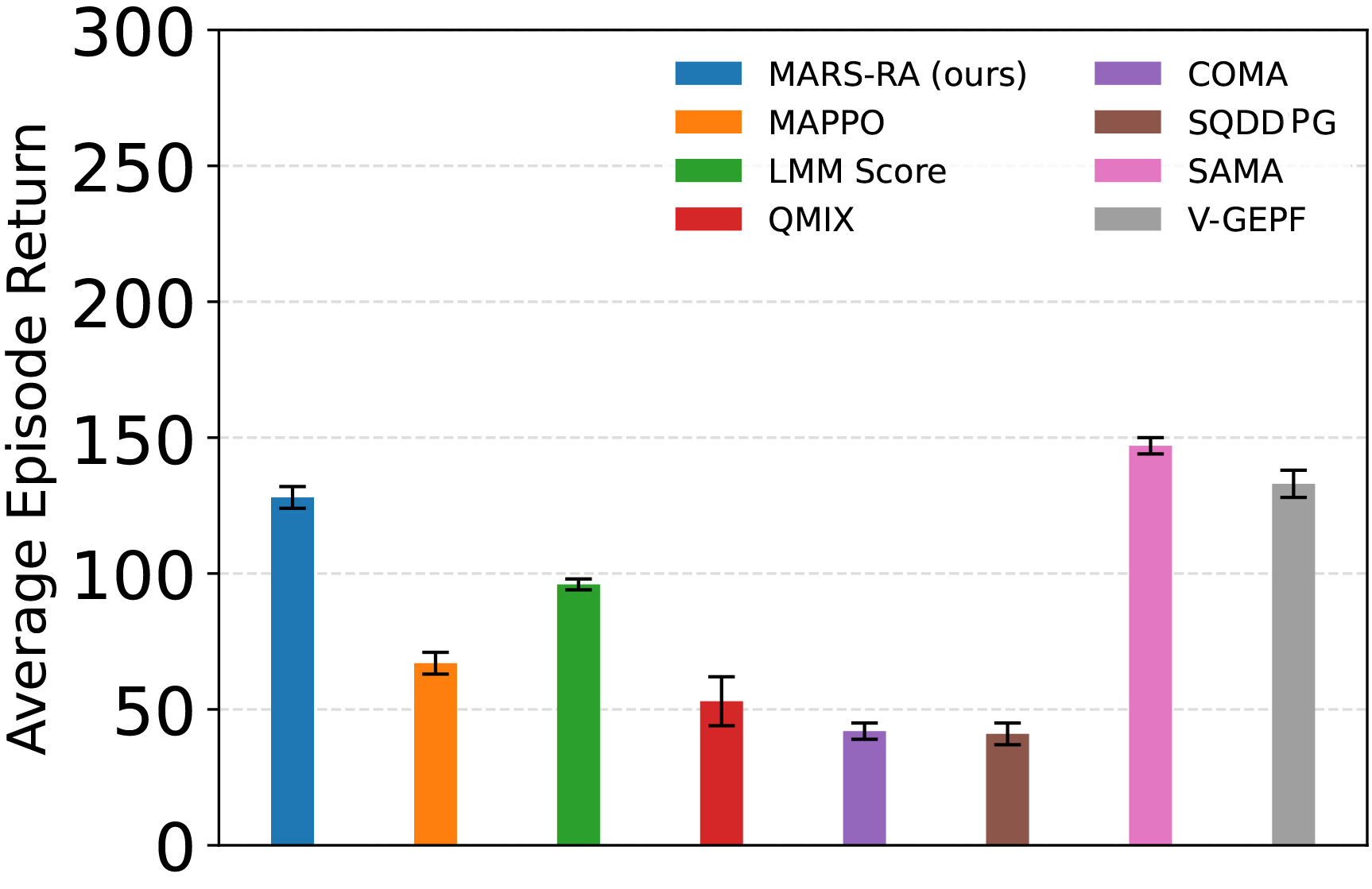}
  \caption{Counter Circuit}
  \label{fig:bar5}
\end{subfigure}
\hfill
\begin{subfigure}{0.32\linewidth}
  \centering
  \includegraphics[width=\linewidth]{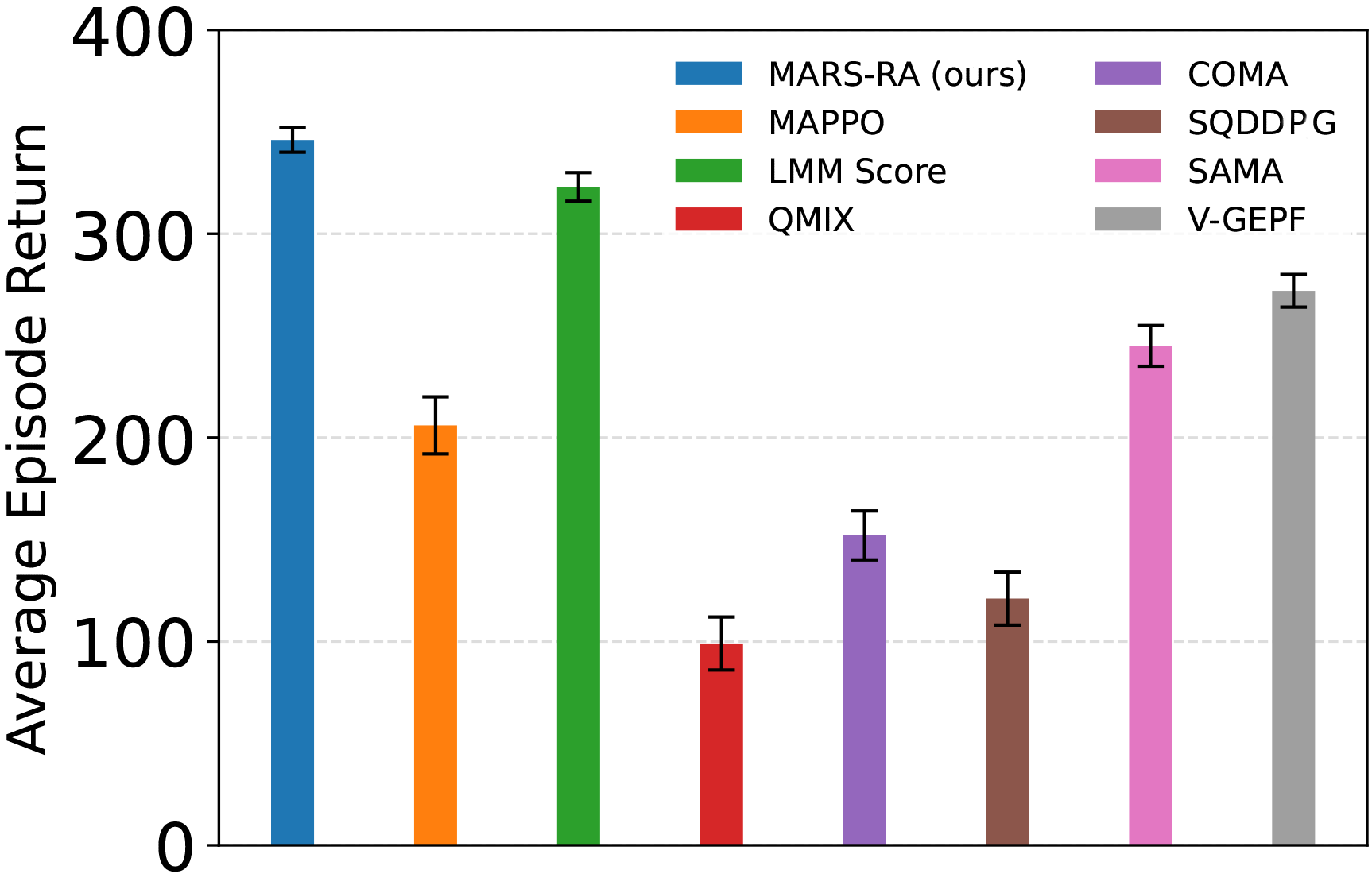}
  \caption{Pistonball}
  \label{fig:bar6}
\end{subfigure}
\caption{Performance comparison on five Overcooked tasks and Pistonball after training for 1 million environment steps, averaged over 10 random seeds, with standard error.}
\label{fig:overcooked_and_positionball}
\end{figure*}

\subsection{Overall Performance on MARS-Bench}
 
Figure~\ref{fig:overall} displays the learning curves for all baselines. MARS-RA surpasses all baselines on three tasks, maintaining success rates above 47\%. In particular, it attains a success rate of exceeding 70\% on the highly collaborative Herd Sheep task. These results indicate that MARS-RA guides effective policy learning across embodied AI cooperative tasks involving 2 to 4 agents.

All selected baselines obtain success rates below 50\% across the three MARS-Bench tasks, highlighting the challenging nature of the benchmark. QMIX and COMA do not exhibit meaningful cooperative behavior during training. SQDDPG shows stronger performance in the Pass Gate task with fewer agents, but its effectiveness diminishes in tasks with more agents, likely due to increased estimation error in Shapley-based credit assignment. Both SAMA and V-GEPF, as state-of-the-art methods, achieve competitive performance but remain consistently inferior to MARS-RA. This performance gap may partly arise from the design assumptions underlying these approaches. Specifically, SAMA’s subgoal formulation is better suited to environments with clearly discretized task structures, whereas embodied AI tasks often lack such explicit decompositions. In contrast, V-GEPF relies on an Adaptive Skill Selection module that assumes timely and informative environmental feedback, while embodied AI tasks typically exhibit delayed feedback and long-horizon dependencies. In this experiment, we computed a 3-cycle rate of 8.95\%, which is substantially lower than the theoretical random baseline of 25.00\%. In addition, the held-out 3-way NLL (win/loss/tie) is 0.41, substantially below the random baseline of 1.10. 

\subsection{Results on Overcooked and Pistonball}
Figures~\ref{fig:overcooked_and_positionball} (a–e) show the evaluation results of the baselines after training on the five Overcooked tasks. MARS-RA achieves performance comparable to the state-of-the-art method SAMA in the Overcooked environment, and outperforms SAMA on the Coordination Ring and Forced Coordination tasks. These results indicate that MARS-RA exhibits stable performance in standard MARL settings such as Overcooked. We also observe that the LMM Score variant performs comparably to MARS-RA on these five tasks, further confirming that its effectiveness improves on tasks with explicit objectives. Figure~\ref{fig:overcooked_and_positionball} (f) illustrates the evaluation performance of the baselines following training on the Pistonball task which involves 10 agents. MARS-RA consistently outperforms all compared baselines, suggesting that the proposed rank-aggregation formulation remains effective beyond small team sizes. We observe that SAMA and V-GEPF achieve relatively modest performance, which may be attributed to the simplicity of the Pistonball task, where there is limited scope for subtask decomposition or the exploitation of more complex strategies.

\subsection{Ablation Study}
As shown in Figure~\ref{fig:overall}, compared to MARS-RA, the backbone algorithm MAPPO alone shows limited ability to learn effective cooperative policies across the three tasks. In contrast, the LMM Score variant learns cooperative behaviors to a limited extent but exhibits substantial performance oscillations, with inferior overall performance. This suggests that the spatio-temporal inference capabilities of LMMs can inform MARL training, whereas direct scoring is less effective than rank aggregation in enabling this translation. Notably, on the Collect Ball task with its explicit objectives, the LMM Score variant achieves relatively competitive performance.






\subsection{Further Analysis}

\begin{figure}
    \centering
    \includegraphics[width=\columnwidth]{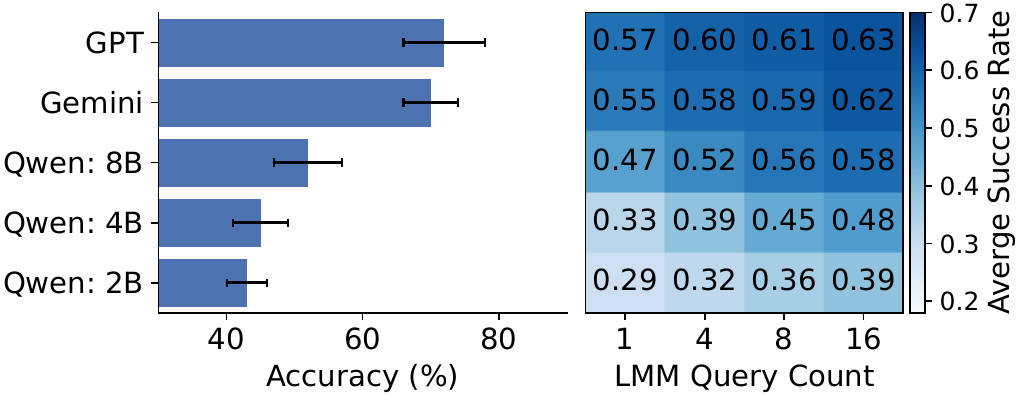}
    \caption{Left: Pairwise comparison accuracy of the selected LMMs, where GPT refers to GPT-5.1, Gemini to Gemini-2.5-Pro, and Qwen to Qwen3-VL. Right: Average success rates of MARS-RA on MARS-Bench under different LMMs and query counts.}
    \label{fig:heatmap}
\end{figure}

\noindent \textbf{Accuracy of LMM-based Pairwise Comparisons.} We further analyze the accuracy of LMM-based pairwise comparisons across a range of popular commercial and open-source LMMs, including Gemini-2.5-Pro, GPT-5.1 \cite{gpt51}, and Qwen3-VL (2B, 4B, and 8B) \cite{yang2025qwen3}. Accuracy is calculated by comparing the LMMs' pairwise comparison results with the ground truth defined by the dense reward function of MARS-Bench, averaged over the three benchmark tasks. As shown in the left panel of Figure~\ref{fig:heatmap}, both Gemini-2.5-Pro and GPT-5.1 achieve accuracy levels above 70\%, while the three smaller-scale Qwen3-VL models also reach accuracies exceeding 43\%. When an agent’s egocentric observation is severely limited and lacks sufficient visual information (e.g., teammate positions or goal locations) to assess team contribution, our evaluation shows that LMMs can still make correct judgments as long as at least one agent’s observation contains informative visual cues. Pairwise comparison errors most frequently occur when all agents face a wall and lack informative visual reference cues. See Appendix~\ref{appendix:acc} for more details.

\noindent \textbf{Model Selection and Number of Queries.} We investigate the impact of LMM selection with different pairwise comparison accuracies and the number of LMM queries used for pairwise comparisons on MARS-RA performance. We train MARS-RA on the three MARS-Bench tasks using Gemini-2.5-Pro, GPT-5.1, and Qwen3-VL (2B, 4B, and 8B) under different numbers of pairwise comparison queries. The average success rates across the three tasks for each configuration are reported in the right panel of Figure~\ref{fig:heatmap}. We observe that both employing LMMs with higher pairwise comparison accuracy and increasing the number of pairwise comparison queries lead to improved MARS-RA performance. However, when using high-accuracy LMMs, the performance gains from increasing the query count are less pronounced compared to those achieved with lower-accuracy LMMs. These results suggest a trade-off and complementary relationship between LMM pairwise comparison accuracy and the number of comparison queries. Further details can be found in Appendix~\ref{appendix:acc}.


\noindent \textbf{Comparison of Rank Aggregation Methods.} We implemented Rank Centrality \cite{negahban2012iterative} as an alternative rank aggregation method under the same experimental settings as in this section, and report the results in Table~\ref{tab:rank_aggregation}. Each value represents the success rate on the corresponding MARS-Bench task achieved by the final trained model. The experimental results show that MARS-RA with Rank Centrality performs slightly worse than with Bradley--Terry model, although the performance gap is small.

\begin{table}[t]
\centering

\setlength{\tabcolsep}{0.8mm}
\begin{tabular}{lccc}
\toprule
Method & Pass Gate & Herd Sheep & Collect Ball \\
\midrule
RC & 0.50 $\pm$ 0.03 & 0.67 $\pm$ 0.02 & 0.46 $\pm$ 0.02 \\
BT              & 0.52 $\pm$ 0.01 & 0.68 $\pm$ 0.03 & 0.46 $\pm$ 0.02 \\
\bottomrule
\end{tabular}
\caption{Comparison of different rank aggregation methods based on average performance across different levels in MARS-Bench. RC denotes Rank Centrality, and BT denotes the Bradley--Terry model.}
\label{tab:rank_aggregation}
\end{table}

\noindent \textbf{Real-World Validation.} We conduct a real-world validation of MARS-RA, as illustrated in Figure~\ref{fig:real_world}. We deploy two XLeRobot robots to perform the Pass Gate task in a real-world indoor environment. The entire room is first captured via 3D scanning and reconstructed as a virtual 3D environment, which is then instantiated as a task in MARS-Bench. MARS-RA is trained in simulation and then deployed on real robots for evaluation in a physical environment, achieving a success rate of 64\% over 25 runs. This result preliminarily suggests the potential of MARS-RA to guide agents toward effective cooperative policies in scenarios with real-world–level complexity. Additional details are provided in Appendix~\ref{appendix:realworld}.

\begin{figure}
    \centering
    \includegraphics[width=\columnwidth]{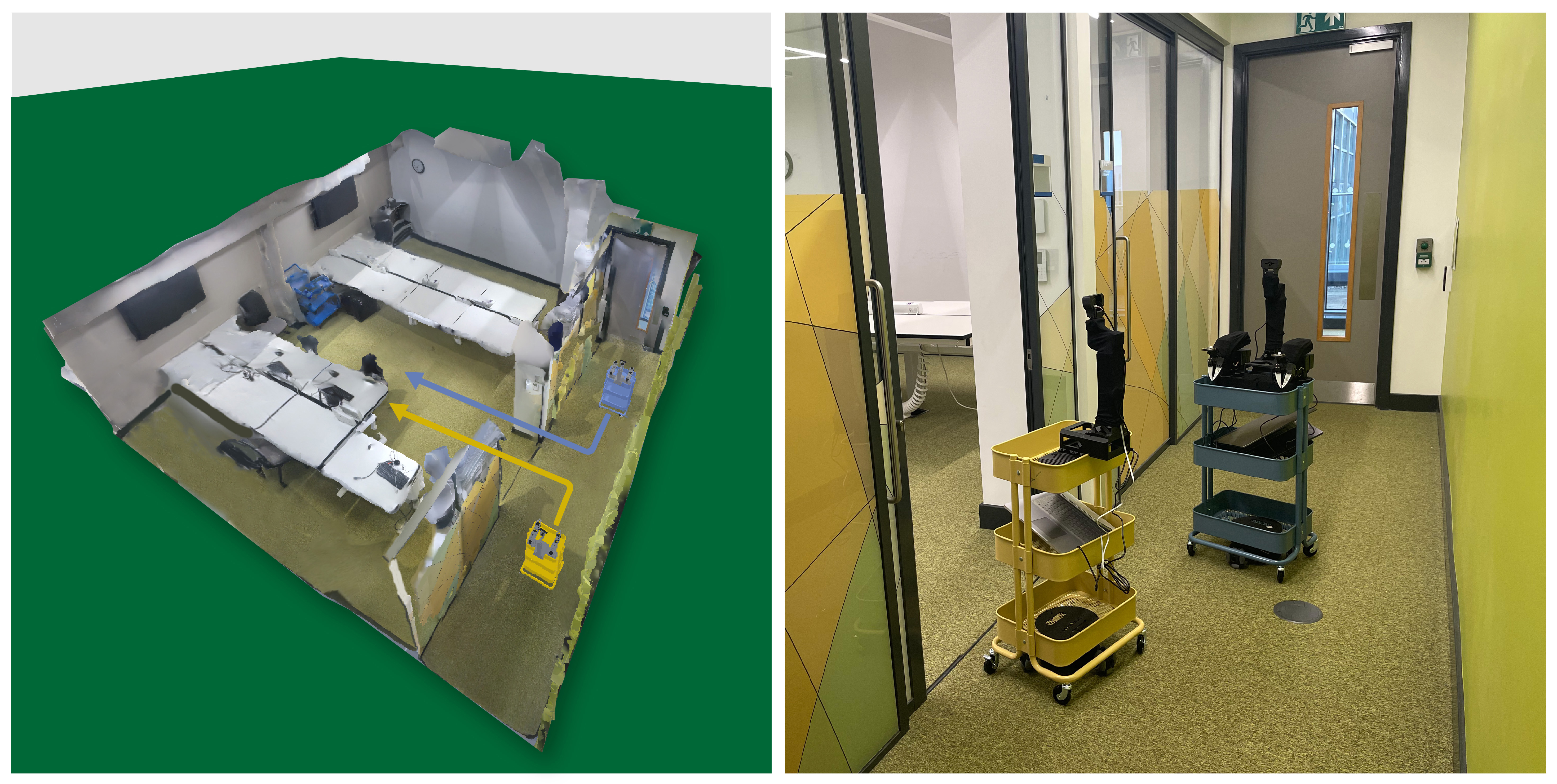}
    \caption{Left: A real-world 3D-scanned indoor environment instantiated as a MARS-Bench Pass Gate task, where two robots enter the room from a corridor without collisions. Right: Real-world deployment after virtual-environment training.}
    \label{fig:real_world}
\end{figure}






\section{Conclusion}
\label{sec:conclusion}

In this paper, we propose MARS-RA, a framework that addresses the credit assignment problem in MARL for cooperative embodied AI systems through a rank aggregation perspective. The framework outperforms strong baselines across different tasks. Beyond empirical validation, we present theoretical analysis to substantiate the proposed framework. This work suggests a new direction for incorporating prior knowledge from foundation models into MARL training.














\section{Limitations}

There exist several avenues for improving this work and mitigating the limitations discussed below: (1) MARS-RA depends on LMMs for pairwise agent comparisons. While rank aggregation helps mitigate noise in these comparisons, the accuracy of LMM-generated judgments remains an important factor affecting overall performance. In MARS-RA, ordered pairwise comparisons are employed to alleviate position bias, and future work will investigate additional intermediate mechanisms to further reduce this reliance. (2) MARS-RA is designed for cooperative tasks whose states and outcomes can be reasonably assessed through visual observations and textual task descriptions. Tasks that rely on fine-grained physical signals, internal states, or domain-specific metrics that are not visually observable may fall outside the current scope of the framework. Future work will explore extensions or alternative formulations better suited to such tasks. (3) Non-stationarity is inherent in MARL, and the pairwise comparisons in MARS-RA do not eliminate this issue. However, the denoising effect of rank aggregation mitigates the impact of non-stationarity on the resulting contribution credits. Moreover, the potential-based reward shaping mechanism ensures that introducing these credits into MARL training preserves policy invariance under stationary transitions. In future work, we will continue to explore new approaches to address non-stationarity.

\bibliography{main}

@book{bishop2006pattern,
  title={Pattern recognition and machine learning},
  author={Bishop, Christopher M and Nasrabadi, Nasser M},
  volume={4},
  number={4},
  year={2006},
  publisher={Springer}
}

@String(AAAI = {AAAI})

@article{turing2021computing,
  title={Computing machinery and intelligence (1950)},
  author={Turing, Alan M},
  journal={Mind},
  volume={59},
  number={236},
  pages={33--60},
  year={2021}
}

@book{clark1998being,
  title={Being there: Putting brain, body, and world together again},
  author={Clark, Andy},
  year={1998},
  publisher={MIT press}
}

@article{feng2025embodied,
  title={Embodied ai: From llms to world models},
  author={Feng, Tongtong and Wang, Xin and Jiang, Yu-Gang and Zhu, Wenwu},
  journal={arXiv preprint arXiv:2509.20021},
  year={2025}
}

@article{minsky2007steps,
  title={Steps toward artificial intelligence},
  author={Minsky, Marvin},
  journal={Proceedings of the IRE},
  volume={49},
  number={1},
  pages={8--30},
  year={2007},
  publisher={IEEE}
}

@article{foerster2017counterfactual,
  title={Counterfactual multi-agent policy gradients (2017)},
  author={Foerster, Jakob and Farquhar, Gregory and Afouras, Triantafyllos and Nardelli, Nantas and Whiteson, Shimon},
  journal={arXiv preprint arXiv:1705.08926},
  year={2017}
}

@article{sunehag2017value,
  title={Value-decomposition networks for cooperative multi-agent learning},
  author={Sunehag, Peter and Lever, Guy and Gruslys, Audrunas and Czarnecki, Wojciech Marian and Zambaldi, Vinicius and Jaderberg, Max and Lanctot, Marc and Sonnerat, Nicolas and Leibo, Joel Z and Tuyls, Karl and others},
  journal={arXiv preprint arXiv:1706.05296},
  year={2017}
}

@article{rashid2020monotonic,
  title={Monotonic value function factorisation for deep multi-agent reinforcement learning},
  author={Rashid, Tabish and Samvelyan, Mikayel and De Witt, Christian Schroeder and Farquhar, Gregory and Foerster, Jakob and Whiteson, Shimon},
  journal={Journal of Machine Learning Research},
  volume={21},
  number={178},
  pages={1--51},
  year={2020}
}

@article{papoudakis2020benchmarking,
  title={Benchmarking multi-agent deep reinforcement learning algorithms in cooperative tasks},
  author={Papoudakis, Georgios and Christianos, Filippos and Sch{\"a}fer, Lukas and Albrecht, Stefano V},
  journal={arXiv preprint arXiv:2006.07869},
  year={2020}
}

@article{nagpal2025leveraging,
  title={Leveraging Large Language Models for Effective and Explainable Multi-Agent Credit Assignment},
  author={Nagpal, Kartik and Dong, Dayi and Bouvier, Jean-Baptiste and Mehr, Negar},
  journal={arXiv preprint arXiv:2502.16863},
  year={2025}
}

@article{wong2023deep,
  title={Deep multiagent reinforcement learning: Challenges and directions},
  author={Wong, Annie and B{\"a}ck, Thomas and Kononova, Anna V and Plaat, Aske},
  journal={Artificial Intelligence Review},
  volume={56},
  number={6},
  pages={5023--5056},
  year={2023},
  publisher={Springer}
}

@article{team2023gemini,
  title={Gemini: a family of highly capable multimodal models},
  author={Team, Gemini and Anil, Rohan and Borgeaud, Sebastian and Alayrac, Jean-Baptiste and Yu, Jiahui and Soricut, Radu and Schalkwyk, Johan and Dai, Andrew M and Hauth, Anja and Millican, Katie and others},
  journal={arXiv preprint arXiv:2312.11805},
  year={2023}
}

@article{liu2023visual,
  title={Visual instruction tuning},
  author={Liu, Haotian and Li, Chunyuan and Wu, Qingyang and Lee, Yong Jae},
  journal={Advances in neural information processing systems},
  volume={36},
  pages={34892--34916},
  year={2023}
}

@article{serra2023learning,
  title={Learning scalable and efficient communication policies for multi-robot collision avoidance},
  author={Serra-G{\'o}mez, {\'A}lvaro and Zhu, Hai and Brito, Bruno and B{\"o}hmer, Wendelin and Alonso-Mora, Javier},
  journal={Autonomous Robots},
  volume={47},
  number={8},
  pages={1275--1297},
  year={2023},
  publisher={Springer}
}

@article{patel2023dream,
  title={Dream: Decentralized reinforcement learning for exploration and efficient energy management in multi-robot systems},
  author={Patel, Dipam and Pham, Phu and Tiwari, Kshitij and Bera, Aniket},
  journal={arXiv preprint arXiv:2309.17433},
  year={2023}
}

@article{yang2025qwen3,
  title={Qwen3 technical report},
  author={Yang, An and Li, Anfeng and Yang, Baosong and Zhang, Beichen and Hui, Binyuan and Zheng, Bo and Yu, Bowen and Gao, Chang and Huang, Chengen and Lv, Chenxu and others},
  journal={arXiv preprint arXiv:2505.09388},
  year={2025}
}

@inproceedings{li2025multi,
  title={Multi-agent credit assignment with pretrained language models},
  author={Li, Wenhao and Qiao, Dan and Wang, Baoxiang and Wang, Xiangfeng and Yin, Wei and Shen, Hao and Jin, Bo and Zha, Hongyuan},
  booktitle={International Conference on Artificial Intelligence and Statistics},
  pages={1945--1953},
  year={2025},
  organization={PMLR}
}

@inproceedings{tian2025identifying,
  title={Identifying and mitigating position bias of multi-image vision-language models},
  author={Tian, Xinyu and Zou, Shu and Yang, Zhaoyuan and Zhang, Jing},
  booktitle={Proceedings of the Computer Vision and Pattern Recognition Conference},
  pages={10599--10609},
  year={2025}
}

@article{brooks1991new,
  title={New approaches to robotics},
  author={Brooks, Rodney A},
  journal={Science},
  volume={253},
  number={5025},
  pages={1227--1232},
  year={1991},
  publisher={American Association for the Advancement of Science}
}

@article{huang2023language,
  title={Language is not all you need: Aligning perception with language models},
  author={Huang, Shaohan and Dong, Li and Wang, Wenhui and Hao, Yaru and Singhal, Saksham and Ma, Shuming and Lv, Tengchao and Cui, Lei and Mohammed, Owais Khan and Patra, Barun and others},
  journal={Advances in Neural Information Processing Systems},
  volume={36},
  pages={72096--72109},
  year={2023}
}

@article{li2024lmeye,
  title={Lmeye: An interactive perception network for large language models},
  author={Li, Yunxin and Hu, Baotian and Chen, Xinyu and Ma, Lin and Xu, Yong and Zhang, Min},
  journal={IEEE Transactions on Multimedia},
  year={2024},
  publisher={IEEE}
}

@article{rocamonde2023vision,
  title={Vision-language models are zero-shot reward models for reinforcement learning},
  author={Rocamonde, Juan and Montesinos, Victoriano and Nava, Elvis and Perez, Ethan and Lindner, David},
  journal={arXiv preprint arXiv:2310.12921},
  year={2023}
}

@inproceedings{ng1999policy,
  title={Policy invariance under reward transformations: Theory and application to reward shaping},
  author={Ng, Andrew Y and Harada, Daishi and Russell, Stuart},
  booktitle={Icml},
  volume={99},
  pages={278--287},
  year={1999},
  organization={Citeseer}
}

@article{lin2025speaking,
  title={Speaking the language of teamwork: Llm-guided credit assignment in multi-agent reinforcement learning},
  author={Lin, Muhan and Shi, Shuyang and Guo, Yue and Tadiparthi, Vaishnav and Chalaki, Behdad and Pari, Ehsan Moradi and Stepputtis, Simon and Kim, Woojun and Campbell, Joseph and Sycara, Katia},
  journal={arXiv preprint arXiv:2502.03723},
  year={2025}
}

@inproceedings{wei2025lero,
  title={LERO: LLM-driven Evolutionary framework with Hybrid Rewards and Enhanced Observation for Multi-Agent Reinforcement Learning},
  author={Wei, Yuan and Shan, Xiaohan and Miao, Ran and Li, Jianmin},
  booktitle={International Conference on Intelligent Computing},
  pages={15--26},
  year={2025},
  organization={Springer}
}

@misc{debreu1960individual,
  title={Individual choice behavior: A theoretical analysis},
  author={Debreu, Gerard},
  year={1960},
  publisher={JSTOR}
}

@article{herbrich2006trueskill,
  title={TrueSkill™: a Bayesian skill rating system},
  author={Herbrich, Ralf and Minka, Tom and Graepel, Thore},
  journal={Advances in neural information processing systems},
  volume={19},
  year={2006}
}

@article{critchlow1991probability,
  title={Probability models on rankings},
  author={Critchlow, Douglas E and Fligner, Michael A and Verducci, Joseph S},
  journal={Journal of mathematical psychology},
  volume={35},
  number={3},
  pages={294--318},
  year={1991},
  publisher={Elsevier}
}

@article{kolde2012robust,
  title={Robust rank aggregation for gene list integration and meta-analysis},
  author={Kolde, Raivo and Laur, Sven and Adler, Priit and Vilo, Jaak},
  journal={Bioinformatics},
  volume={28},
  number={4},
  pages={573--580},
  year={2012},
  publisher={Oxford University Press}
}

@inproceedings{di2025balancing,
  title={Balancing invariant and specific knowledge for domain generalization with online knowledge distillation},
  author={Di Zhao, Jingfeng Zhang and Hu, Hongsheng and Fournier-Viger, Philippe and Dobbie, Gillian and Koh, Yun Sing},
  booktitle={Proceedings of the Thirty-Fourth International Joint Conference on Artificial Intelligence, IJCAI-25},
  pages={2440--2448},
  year={2025}
}

@article{wang2026towermind,
  title={TowerMind: A Tower Defence Game Learning Environment and Benchmark for LLM as Agents},
  author={Wang, Dawei and Zhou, Chengming and Zhao, Di and Liu, Xinyuan and Ma, Marci Chi and Ushaw, Gary and Davison, Richard},
  journal={arXiv preprint arXiv:2601.05899},
  year={2026}
}

@inproceedings{zhao2024symmetric,
  title={Symmetric self-paced learning for domain generalization},
  author={Zhao, Di and Koh, Yun Sing and Dobbie, Gillian and Hu, Hongsheng and Fournier-Viger, Philippe},
  booktitle={Proceedings of the AAAI Conference on Artificial Intelligence},
  volume={38},
  number={15},
  pages={16961--16969},
  year={2024}
}

@article{li2026videothinker,
  title={VideoThinker: Building Agentic VideoLLMs with LLM-Guided Tool Reasoning},
  author={Li, Chenglin and Chen, Qianglong and Han, Feng and Wang, Yikun and Yin, Xingxi and Gong, Yan and Li, Ruilin and Zhang, Yin and Wang, Jiaqi},
  journal={arXiv preprint arXiv:2601.15724},
  year={2026}
}

@article{li2025videopro,
  title={VideoPro: Adaptive Program Reasoning for Long Video Understanding},
  author={Li, Chenglin and Han, Feng and Wang, Yikun and Li, Ruilin and Dong, Shuai and Hou, Haowen and Li, Haitao and Chen, Qianglong and Tao, Feng and Tong, Jingqi and others},
  journal={arXiv preprint arXiv:2509.17743},
  year={2025}
}

@article{dong2025interleaved,
  title={Interleaved latent visual reasoning with selective perceptual modeling},
  author={Dong, Shuai and Wang, Siyuan and Liu, Xingyu and Li, Chenglin and Hou, Haowen and Wei, Zhongyu},
  journal={arXiv preprint arXiv:2512.05665},
  year={2025}
}

@inproceedings{
zhao2026unlearning,
title={Unlearning during Training: Domain-Specific Gradient Ascent for Domain Generalization},
author={Di Zhao and Jingfeng Zhang and Hongsheng Hu and Philippe Fournier-Viger and Gillian Dobbie and Yun Sing Koh},
booktitle={The Fourteenth International Conference on Learning Representations},
year={2026},
url={https://openreview.net/forum?id=9ufS5Jl0O0}
}

@article{negahban2012iterative,
  title={Iterative ranking from pair-wise comparisons},
  author={Negahban, Sahand and Oh, Sewoong and Shah, Devavrat},
  journal={Advances in neural information processing systems},
  volume={25},
  year={2012}
}

@article{ma2022tale,
  title={A tale of hodgerank and spectral method: Target attack against rank aggregation is the fixed point of adversarial game},
  author={Ma, Ke and Xu, Qianqian and Zeng, Jinshan and Li, Guorong and Cao, Xiaochun and Huang, Qingming},
  journal={IEEE Transactions on Pattern Analysis and Machine Intelligence},
  volume={45},
  number={4},
  pages={4090--4108},
  year={2022},
  publisher={IEEE}
}

@article{abadi2025challenges,
  title={Challenges in Credit Assignment for Multi-Agent Reinforcement Learning in Open Agent Systems},
  author={Abadi, Alireza Saleh and Soh, Leen-Kiat},
  journal={arXiv preprint arXiv:2510.27659},
  year={2025}
}

@inproceedings{tang2023roma,
  title={RoMA: Resilient Multi-Agent Reinforcement Learning with Dynamic Participating Agents},
  author={Tang, Xuting and Xu, Jia and Wang, Shusen},
  booktitle={2023 IEEE 12th International Conference on Cloud Networking (CloudNet)},
  pages={247--255},
  year={2023},
  organization={IEEE}
}

@article{tao2024maniskill3,
  title={Maniskill3: Gpu parallelized robotics simulation and rendering for generalizable embodied ai},
  author={Tao, Stone and Xiang, Fanbo and Shukla, Arth and Qin, Yuzhe and Hinrichsen, Xander and Yuan, Xiaodi and Bao, Chen and Lin, Xinsong and Liu, Yulin and Chan, Tse-kai and others},
  journal={arXiv preprint arXiv:2410.00425},
  year={2024}
}

@article{shapley1953value,
  title={A value for n-person games},
  author={Shapley, Lloyd S and others},
  year={1953},
  publisher={Princeton University Press Princeton}
}

@inproceedings{cohen2017open,
  title={Open decentralized pomdps},
  author={Cohen, Jonathan and Dibangoye, Jilles-Steeve and Mouaddib, Abdel-Illah},
  booktitle={2017 IEEE 29th International Conference on Tools with Artificial Intelligence (ICTAI)},
  pages={977--984},
  year={2017},
  organization={IEEE}
}

@inproceedings{devlin2012dynamic,
  title={Dynamic potential-based reward shaping},
  author={Devlin, Sam Michael and Kudenko, Daniel},
  booktitle={11th International Conference on Autonomous Agents and Multiagent Systems (AAMAS 2012)},
  pages={433--440},
  year={2012},
  organization={IFAAMAS}
}

@article{christiano2017deep,
  title={Deep reinforcement learning from human preferences},
  author={Christiano, Paul F and Leike, Jan and Brown, Tom and Martic, Miljan and Legg, Shane and Amodei, Dario},
  journal={Advances in neural information processing systems},
  volume={30},
  year={2017}
}

@misc{wang2025xlerobot,
    author = {Wang, Gaotian and Lu, Zhuoyi},
    title = {XLeRobot: A Practical Low-cost Household Dual-Arm Mobile Robot Design for General Manipulation},
    howpublished = "\url{https://github.com/Vector-Wangel/XLeRobot}",
    year = {2025}
}

@article{wu2021too,
  title={Too many cooks: Bayesian inference for coordinating multi-agent collaboration},
  author={Wu, Sarah A and Wang, Rose E and Evans, James A and Tenenbaum, Joshua B and Parkes, David C and Kleiman-Weiner, Max},
  journal={Topics in Cognitive Science},
  volume={13},
  number={2},
  pages={414--432},
  year={2021},
  publisher={Wiley Online Library}
}

@inproceedings{xiong2024mqe,
  title={Mqe: Unleashing the power of interaction with multi-agent quadruped environment},
  author={Xiong, Ziyan and Chen, Bo and Huang, Shiyu and Tu, Wei-Wei and He, Zhaofeng and Gao, Yang},
  booktitle={2024 IEEE/RSJ International Conference on Intelligent Robots and Systems (IROS)},
  pages={5918--5924},
  year={2024},
  organization={IEEE}
}

@article{bradley1952rank,
  title={Rank analysis of incomplete block designs: I. the method of paired comparisons},
  author={Bradley, Ralph Allan and Terry, Milton E},
  journal={Biometrika},
  volume={39},
  number={3/4},
  pages={324--345},
  year={1952},
  publisher={JSTOR}
}

@inproceedings{wierstra2008episodic,
  title={Episodic reinforcement learning by logistic reward-weighted regression},
  author={Wierstra, Daan and Schaul, Tom and Peters, Jan and Schmidhuber, Juergen},
  booktitle={International Conference on Artificial Neural Networks},
  pages={407--416},
  year={2008},
  organization={Springer}
}

@article{shi2024judging,
  title={Judging the judges: A systematic study of position bias in llm-as-a-judge},
  author={Shi, Lin and Ma, Chiyu and Liang, Wenhua and Diao, Xingjian and Ma, Weicheng and Vosoughi, Soroush},
  journal={arXiv preprint arXiv:2406.07791},
  year={2024}
}

@inproceedings{harutyunyan2015expressing,
  title={Expressing arbitrary reward functions as potential-based advice},
  author={Harutyunyan, Anna and Devlin, Sam and Vrancx, Peter and Now{\'e}, Ann},
  booktitle={Proceedings of the AAAI conference on artificial intelligence},
  volume={29},
  number={1},
  year={2015}
}

@article{zhang2024combo,
  title={COMBO: compositional world models for embodied multi-agent cooperation},
  author={Zhang, Hongxin and Wang, Zeyuan and Lyu, Qiushi and Zhang, Zheyuan and Chen, Sunli and Shu, Tianmin and Dariush, Behzad and Lee, Kwonjoon and Du, Yilun and Gan, Chuang},
  journal={arXiv preprint arXiv:2404.10775},
  year={2024}
}

@book{bakushinsky2012ill,
  title={Ill-posed problems: theory and applications},
  author={Bakushinsky, Anatoly and Goncharsky, A},
  volume={301},
  year={2012},
  publisher={Springer Science \& Business Media}
}

@article{negahban2012unified,
  title={A unified framework for high-dimensional analysis of $M$-estimators with decomposable regularizers},
  author={Negahban, Sahand N and Ravikumar, Pradeep and Wainwright, Martin J and Yu, Bin},
  journal={Statistical Science},
  volume={27},
  number={4},
  pages={538--557},
  year={2012}
}

@article{shah2016estimation,
  title={Estimation from pairwise comparisons: Sharp minimax bounds with topology dependence},
  author={Shah, Nihar B and Balakrishnan, Sivaraman and Bradley, Joseph and Parekh, Abhay and Ramchandran, Kannan and Wainwright, Martin J},
  journal={Journal of Machine Learning Research},
  volume={17},
  number={58},
  pages={1--47},
  year={2016}
}

@article{yu2022surprising,
  title={The surprising effectiveness of ppo in cooperative multi-agent games},
  author={Yu, Chao and Velu, Akash and Vinitsky, Eugene and Gao, Jiaxuan and Wang, Yu and Bayen, Alexandre and Wu, Yi},
  journal={Advances in neural information processing systems},
  volume={35},
  pages={24611--24624},
  year={2022}
}

@article{comanici2025gemini,
  title={Gemini 2.5: Pushing the frontier with advanced reasoning, multimodality, long context, and next generation agentic capabilities},
  author={Comanici, Gheorghe and Bieber, Eric and Schaekermann, Mike and Pasupat, Ice and Sachdeva, Noveen and Dhillon, Inderjit and Blistein, Marcel and Ram, Ori and Zhang, Dan and Rosen, Evan and others},
  journal={arXiv preprint arXiv:2507.06261},
  year={2025}
}

@inproceedings{ma2025vision,
  title={Vision-Based Generic Potential Function for Policy Alignment in Multi-Agent Reinforcement Learning},
  author={Ma, Hao and Wang, Shijie and Pu, Zhiqiang and Zhao, Siyao and Ai, Xiaolin},
  booktitle={Proceedings of the AAAI Conference on Artificial Intelligence},
  volume={39},
  number={18},
  pages={19287--19295},
  year={2025}
}

@inproceedings{wang2020shapley,
  title={Shapley Q-value: A local reward approach to solve global reward games},
  author={Wang, Jianhong and Zhang, Yuan and Kim, Tae-Kyun and Gu, Yunjie},
  booktitle={Proceedings of the AAAI conference on artificial intelligence},
  volume={34},
  number={05},
  pages={7285--7292},
  year={2020}
}

@misc{gpt51,
  author = {OpenAI},
  title = {Gpt-5 system card},
  year = {2025},
  howpublished = {\url{ https://cdn.openai.com/gpt-5-system-card.pdf}}
}

@article{zhang2021multi,
  title={Multi-agent reinforcement learning: A selective overview of theories and algorithms},
  author={Zhang, Kaiqing and Yang, Zhuoran and Ba{\c{s}}ar, Tamer},
  journal={Handbook of reinforcement learning and control},
  pages={321--384},
  year={2021},
  publisher={Springer}
}

@article{liu2023xuance,
  title={XuanCe: A comprehensive and unified deep reinforcement learning library},
  author={Liu, Wenzhang and Cai, Wenzhe and Jiang, Kun and Cheng, Guangran and Wang, Yuanda and Wang, Jiawei and Cao, Jingyu and Xu, Lele and Mu, Chaoxu and Sun, Changyin},
  journal={arXiv preprint arXiv:2312.16248},
  year={2023}
}

@article{wang2022shaq,
  title={Shaq: Incorporating shapley value theory into multi-agent q-learning},
  author={Wang, Jianhong and Zhang, Yuan and Gu, Yunjie and Kim, Tae-Kyun},
  journal={Advances in Neural Information Processing Systems},
  volume={35},
  pages={5941--5954},
  year={2022}
}

@article{terry2021pettingzoo,
  title={Pettingzoo: Gym for multi-agent reinforcement learning},
  author={Terry, J and Black, Benjamin and Grammel, Nathaniel and Jayakumar, Mario and Hari, Ananth and Sullivan, Ryan and Santos, Luis S and Dieffendahl, Clemens and Horsch, Caroline and Perez-Vicente, Rodrigo and others},
  journal={Advances in Neural Information Processing Systems},
  volume={34},
  pages={15032--15043},
  year={2021}
}

@article{carroll2019utility,
  title={On the utility of learning about humans for human-ai coordination},
  author={Carroll, Micah and Shah, Rohin and Ho, Mark K and Griffiths, Tom and Seshia, Sanjit and Abbeel, Pieter and Dragan, Anca},
  journal={Advances in neural information processing systems},
  volume={32},
  year={2019}
}

\newpage

\appendix

\section{Additional Details of MARS-Bench}
\label{appendix:mars-bench}

Each agent receives an egocentric pixel-based observation. The action space is discrete and consists of five actions: \textit{no-operation}, \textit{move forward}, \textit{move backward}, \textit{turn left}, and \textit{turn right}. The reward signal is defined under two settings: a sparse and a dense mode. In the sparse reward setting, the agent receives a reward of $+1$ upon task completion and a penalty of $-1$ for collisions or timeouts. In the dense reward setting, agents receive per-agent, step-wise shaping rewards that reflect progress toward task completion. The sparse reward mode is designed to closely reflect real-world conditions during training. In contrast, the dense reward mode leverages extensive environment-internal variables that would be unavailable in real-world settings, and is used solely for analysis and debugging purposes. Each episode has a maximum length of 2000 steps and terminates immediately upon either success or failure. Below we provide detailed descriptions of the three tasks:

\begin{itemize}[noitemsep, topsep=0pt, leftmargin=*]
    \item \textbf{Pass Gate}: This task involves 2 agents. It presents a bottleneck challenge in which a narrow gate divides the environment. Two agents starting on the same side must resolve contention for access to this shared passage, necessitating policies that can manage spatio-temporal conflicts.
    \item \textbf{Herd Sheep}: This task involves 3 agents. We introduce sheep as non-player characters whose autonomous behavior combines avoidance of agents, attraction to the flock's centroid, and stochastic wandering. The agents must collaborate to drive all sheep through the gate into the next room. This task tests the agents’ ability to understand the sheep’s movement dynamics and to cooperate in executing the herding task.
    \item \textbf{Collect Ball}: This task involves 4 agents. Agents must collect all balls in the room, aiming to maximize efficiency via parallel execution.
\end{itemize}

\section{Detailed Proofs of Theoretical Properties}
\label{appendix:proofs}

In this appendix, we provide the step-by-step mathematical derivations for the theoretical properties presented in Section \ref{sec:theory}.

\subsection{Proof of Proposition \ref{pro:robustness} (Convergence and Robustness)}

\textbf{Proposition Restatement:} \textit{Suppose the comparison graph is connected with algebraic connectivity $\lambda_2 > 0$. Let $\hat{\mathbf{c}_t}$ be the MLE estimate derived from $K$ pairwise comparisons. With probability at least $1 - n^{-2}$, the estimation error satisfies $\|\hat{\mathbf{c}_t} - \mathbf{c}_t^*\|_2 \le \mathcal{O}(\sqrt{\frac{n \log n}{K}})$.}

\begin{proof}
We analyze the error bound using the framework of M-estimation in convex optimization. The proof proceeds in three main steps: (1) establishing the strong convexity of the loss function, (2) bounding the gradient at the optimum, and (3) deriving the parameter error bound.

\paragraph{Step 1: Negative Log-Likelihood and Convexity.}
The negative log-likelihood (loss function) for the Bradley–Terry model with $K$ samples is:
\begin{equation}
\begin{split}
    \mathcal{L}(\mathbf{c}_t) = \sum_{k=1}^K & \left[ \log(1 + e^{c_{t,i_k} - c_{t, j_k}}) \right. \\
    & \left. - y_k(c_{t, i_k} - c_{t, j_k}) \right]
\end{split}
\end{equation}
where $y_k=1$ if $i_k \succ j_k$ and $0$ otherwise.
The Hessian matrix of the loss function, $\mathbf{H}(\mathbf{c}_t) = \nabla^2 \mathcal{L}(\mathbf{c}_t)$, corresponds to the Laplacian of the comparison graph weighted by the variances of the logistic distribution. For a connected graph, the smallest non-zero eigenvalue of the Laplacian, denoted as $\lambda_2$ (algebraic connectivity), is strictly positive.
Thus, restricted to the subspace orthogonal to the constant vector $\mathbf{1}$ (to handle translation invariance), the loss function is $\mu$-strongly convex locally around $\mathbf{c}_t^*$:
\begin{equation}
    \Delta^T \nabla^2 \mathcal{L}(\mathbf{c}_t) \Delta \ge \mu \|\Delta\|_2^2, \quad \forall \Delta \in \mathbb{R}^n, \Delta \perp \mathbf{1}
    \label{eq:strong_convexity}
\end{equation}
where $\mu$ scales linearly with the number of samples $K$ (assuming uniform sampling of pairs).

\paragraph{Step 2: Taylor Expansion.}
Since $\hat{\mathbf{c}_t}$ minimizes $\mathcal{L}(\mathbf{c}_t)$, the gradient $\nabla \mathcal{L}(\hat{\mathbf{c}_t}) = 0$. We verify the error $\Delta = \hat{\mathbf{c}_t} - \mathbf{c}_t^*$ using the first-order Taylor expansion of the gradient around $\mathbf{c}_t^*$:
\begin{equation}
    \nabla \mathcal{L}(\hat{\mathbf{c}_t}) \approx \nabla \mathcal{L}(\mathbf{c}_t^*) + \nabla^2 \mathcal{L}(\mathbf{c}_t^*) (\hat{\mathbf{c}_t} - \mathbf{c}_t^*)
\end{equation}
Setting $\nabla \mathcal{L}(\hat{\mathbf{c}_t}) = 0$ and multiplying by $\Delta^T$:
\begin{equation}
    0 \approx \Delta^T \nabla \mathcal{L}(\mathbf{c}_t^*) + \Delta^T \nabla^2 \mathcal{L}(\mathbf{c}_t^*) \Delta
\end{equation}
Using the strong convexity property (Eq. \ref{eq:strong_convexity}) and the Cauchy-Schwarz inequality:
\begin{equation}
\begin{split}
    \mu \|\Delta\|_2^2 \le \Delta^T \nabla^2 \mathcal{L}(\mathbf{c}_t^*) \Delta 
    & = - \Delta^T \nabla \mathcal{L}(\mathbf{c}_t^*) \\
    & \le \|\Delta\|_2 \|\nabla \mathcal{L}(\mathbf{c}_t^*)\|_2
\end{split}
\end{equation}
Dividing both sides by $\mu \|\Delta\|_2$:
\begin{equation}
    \|\Delta\|_2 \le \frac{1}{\mu} \|\nabla \mathcal{L}(\mathbf{c}_t^*)\|_2
    \label{eq:error_bound_raw}
\end{equation}

\paragraph{Step 3: Bounding the Gradient Norm (Concentration).}
The gradient at the true parameter $\mathbf{c}_t^*$ is given by:
\begin{equation}
    \nabla \mathcal{L}(\mathbf{c}_t^*) = \sum_{k=1}^K (P_{i_k j_k}^* - y_k) \mathbf{x}_k
\end{equation}
where $\mathbf{x}_k$ is the indicator vector for the pair $(i_k, j_k)$. Since $E[y_k] = P_{i_k j_k}^*$, the expected gradient is $\mathbb{E}[\nabla \mathcal{L}(\mathbf{c}_t^*)] = \mathbf{0}$.
The term $(P^* - y_k)$ is a bounded random variable in $[-1, 1]$. By Hoeffding's inequality (or Azuma-Hoeffding for martingales), the norm of the sum of these random variables is bounded with high probability. Specifically, for $n$ dimensions:
\begin{equation}
    \mathbb{P}\left( \|\nabla \mathcal{L}(\mathbf{c}_t^*)\|_2 \ge \tau \right) \le 2n \exp\left( - \frac{\tau^2}{C K} \right)
\end{equation}
Setting the probability to $n^{-2}$, we get the bound $\|\nabla \mathcal{L}(\mathbf{c}_t^*)\|_2 \le \mathcal{O}(\sqrt{K \log n})$.

\paragraph{Final Assembly.}
Substituting the gradient bound ($\sqrt{K}$) and the strong convexity parameter ($\mu \propto K$) into Eq. (\ref{eq:error_bound_raw}):
\begin{equation}
    \|\hat{\mathbf{c}_t} - \mathbf{c}_t^*\|_2 \le \frac{\mathcal{O}(\sqrt{K \log n})}{\mathcal{O}(K)} = \mathcal{O}\left( \sqrt{\frac{\log n}{K}} \right)
\end{equation}
Normalizing by dimension (as per the theorem statement) yields the stated result.
\end{proof}

\subsection{Proof of Proposition 2 (Alignment with Shapley Values)}

\textbf{Proposition Restatement:} \textit{If $\mathbf{c}_t^* = \boldsymbol{v}_t^*$, then $\hat{\mathbf{c}_t}$ derived by MLE are consistent estimators of the true Shapley values.}

\begin{proof}
This proof relies on the consistency of Maximum Likelihood Estimators and the Law of Large Numbers.

The derivative of the log-likelihood function with respect to the score $c_{t,i}$ is:
\begin{equation}
    \frac{\partial \ell}{\partial c_{t,i}} = \sum_{j \neq i} \left( n_{ij} - N_{ij} \frac{e^{c_{t,i}}}{e^{c_{t,i}} + e^{c_{t,j}}} \right)
\end{equation}
where $n_{ij}$ is the observed number of times $i$ beats $j$, and $N_{ij}$ is the total comparisons. The MLE solution $\hat{\mathbf{c}_t}$ satisfies $\frac{\partial \ell}{\partial c_{t,i}} = 0$ for all $i$.

Dividing the equation by the total number of samples $K$ and taking the limit as $K \to \infty$:
\begin{equation}
    \lim_{K \to \infty} \frac{n_{ij}}{N_{ij}} = \mathbb{P}(i \succ j \mid \mathbf{c}_t^*) = \frac{e^{c^*_{t,i}}}{e^{c^*_{t,i}} + e^{c^*_{t,j}}}
\end{equation}
The optimality condition for the estimator $\hat{\mathbf{c}_t}$ in the limit becomes:
\begin{equation}
    \sum_{j \neq i} \pi_{ij} \left( \frac{e^{c^*_{t,i}}}{e^{c^*_{t,i}} + e^{c^*_{t,j}}} - \frac{e^{\hat{c}_{t,i}}}{e^{\hat{c}_{t,i}} + e^{\hat{c}_{t,j}}} \right) = 0
\end{equation}
where $\pi_{ij}$ is the sampling probability of pair $(i, j)$.
Since the logistic function $\sigma(x) = \frac{1}{1+e^{-x}}$ is strictly monotonic, the only solution to this system of equations (assuming a connected graph) implies:
\begin{equation}
    \hat{c}_{t,i} - \hat{c}_{t,j} = c^*_{t,i} - c^*_{t,j}, \quad \forall i, j
\end{equation}
Thus, $\hat{\mathbf{c}_t} = \mathbf{c}_t^* + C$. Since we assume $\mathbf{c}_t^* = \boldsymbol{v}^*$, the estimated scores converge to the Shapley values.
\end{proof}

\section{Experimental Details}
\label{appendix:exp_details}

We employ action repeat during training, where agents make decisions once every 32 environment steps. From each action-repeat window, we uniformly sample four visual frames, convert them to grayscale, and stack them to form a $128 \times 128 \times 4$ observation space for each agent. All training models in our experiments adopt a unified multi-layer CNN-based architecture, as illustrated in Figure~\ref{fig:model}. MAPPO, QMIX, and COMA are implemented using the XuanCe \cite{liu2023xuance} framework, while the implementations of SQDDPG, SAMA, and V-GEPF are obtained from their respective open-source code repositories. All MAPPO-based methods, including MARS-RA, MAPPO, LMM Score, SAMA, and V-GEPF, share the same MAPPO hyperparameters, as reported in Table~\ref{tab:ppo_hparams}. The hyperparameters for QMIX and COMA follow those used in \citet{papoudakis2020benchmarking}, the hyperparameters for SQDDPG follow those used in \citet{wang2022shaq}. For a fair comparison, SAMA and V-GEPF employ the same LMM as MARS-RA (Gemini-2.5-Pro), and all other configurations follow the original implementations.

All experiments were conducted on a desktop computer running Ubuntu 24.04, equipped with an Intel(R) Core(TM) i7-14700K CPU (20 cores), an NVIDIA GeForce RTX 3090 GPU with 24 GB of VRAM, and 128 GB of RAM. For GPT-5.1 and Gemini-2.5-Pro, we use the official API services provided on their respective websites. In contrast, Qwen3-VL (2B, 4B, and 8B) is deployed locally. To accelerate LMM-based pairwise comparisons during training, we deploy Qwen3-VL across 4 desktop machines, each equipped with an NVIDIA GeForce RTX 3090 GPU. Under our experimental setup and computational resources, training MAPPO for 50 million environment steps completes within 16 hours, while training MARS-RA for 50 million steps completes within 30 hours.

Moreover, we conduct an additional hyperparameter sweep over $\rho$ on MARS-Bench to investigate how different $\rho$ values affect the performance of MARS-RA. The results are reported in Table~\ref{tab:rho_ablation}, where each value denotes the success rate achieved by the final trained model on the corresponding MARS-Bench task. The hyperparameter sweep results indicate that MARS-RA achieves its best performance when $\rho \in \{0.1, 0.5, 1.0\}$. When $\rho$ takes a smaller value (e.g., 0.03) or a larger value (e.g., 5), the performance of MARS-RA declines and the standard error increases. When $\rho$ = 0, MARS-RA reduces to MAPPO, and its performance is consistent with that of MAPPO.

\begin{table}[ht]
  \centering
  \begin{tabular}{@{}lll@{}}
    \toprule
    \textbf{Hyperparameter}                       &  \textbf{MAPPO}  \\
    \midrule
    Optimizer                                 &  Adam \\
    Learning rate                             &  0.001             \\
    Training Batch Size                    &  10240                    \\
    Minibatch Size                        &    512                        \\
    Discount factor ($\gamma$)                &  0.99                   \\
    GAE ($\lambda$)                            &  0.95                   \\
    Policy clip ratio                         &  0.20                    \\
    Epochs                  &  10                      \\
    \bottomrule
  \end{tabular}
  \caption{Key hyperparameters for MAPPO in our experiments. }
  \label{tab:ppo_hparams}
\end{table}

\begin{figure}
    \centering
    \includegraphics[width=\columnwidth]{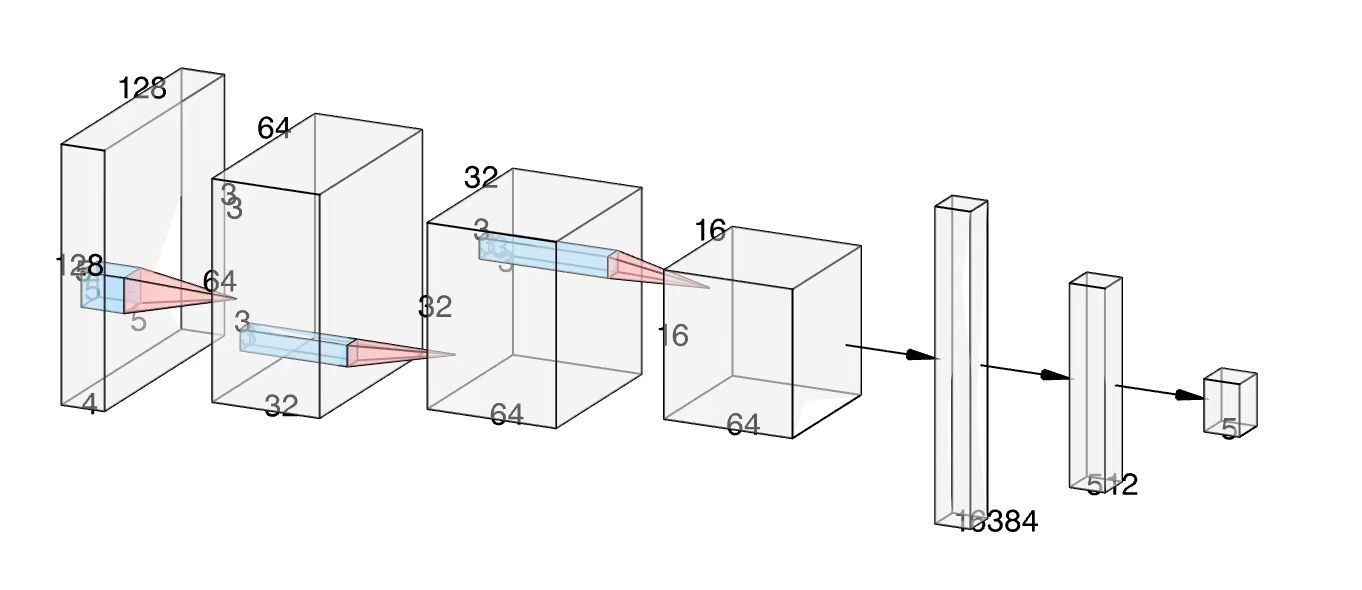}
    \caption{Model architectures used by all MARL algorithms in this study, a CNN followed by MLP.}
    \label{fig:model}
\end{figure}

\begin{table}[t]
\setlength{\tabcolsep}{1.0mm}
\centering

\begin{tabular}{lccc}
\toprule
$\rho$ & Pass Gate & Herd Sheep & Collect Ball \\
\midrule
$\rho = 0$    & 0.02 $\pm$ 0.00 & 0.06 $\pm$ 0.02 & 0.03 $\pm$ 0.01 \\
$\rho = 0.03$ & 0.35 $\pm$ 0.03 & 0.52 $\pm$ 0.04 & 0.33 $\pm$ 0.06 \\
$\rho = 0.1$  & 0.50 $\pm$ 0.03 & 0.66 $\pm$ 0.02 & 0.44 $\pm$ 0.03 \\
$\rho = 0.5$  & 0.53 $\pm$ 0.02 & 0.66 $\pm$ 0.01 & 0.47 $\pm$ 0.04 \\
$\rho = 1.0$  & 0.52 $\pm$ 0.01 & 0.68 $\pm$ 0.03 & 0.46 $\pm$ 0.02 \\
$\rho = 5.0$  & 0.40 $\pm$ 0.07 & 0.51 $\pm$ 0.05 & 0.39 $\pm$ 0.04 \\
\bottomrule
\end{tabular}
\caption{Effect of different $\rho$ values on performance across levels in MARS-Bench.}
\label{tab:rho_ablation}
\end{table}

\section{Experiments on Overcooked and Pistonball}
\label{appendix:more_env}

We evaluate MARS-RA and the selected baselines on the widely used MARL environments Overcooked \cite{carroll2019utility} and Pistonball \cite{terry2021pettingzoo}. This evaluation aims to verify that MARS-RA can guide agents to learn effective cooperative policies not only in embodied AI scenarios, but also in classic MARL tasks. We retain the original observation spaces, discrete action spaces, and reward functions provided by these environments, without introducing dynamically varying numbers of active agents or using action repeat. In both environments, models are trained for 1 million environment steps. All other experimental settings follow those in Section~\ref{sec:experiment} and Appendix~\ref{appendix:exp_details}.

\subsection{Overcooked}
In this kitchen cooking game, two agents must collaboratively prepare and deliver soup to obtain a shared team reward, as illustrated in Figure~\ref{fig:overcooked}. In this environment, agents receive a team reward of 20 upon each successful soup delivery, and the episode return is used as the evaluation metric. It includes five tasks:

\begin{itemize}
  \item \textbf{Cramped Room}: The setting is a restrictive room where two agents must share a single pot and serving point. The challenge encourages the agents to fully exploit the limited resources to cook and serve soup, achievable through simple cooperative strategies.
  \item \textbf{Asymmetric Advantages}: Players operate in two separate, isolated kitchens with an asymmetric layout. On the left side, onions are far from the pots, but serving points are centrally located. On the right side, the setup is reversed: onions are near the center, while serving points are far away.
  \item \textbf{Coordination Ring}: This circular design forces players to keep moving to avoid collisions, especially at the choke points in the top-right and bottom-left corners housing the ingredients and pots. Success depends on the simultaneous use of both pots.
  \item \textbf{Forced Coordination}: This layout creates a dependency loop by isolating the agents. Since the left side lacks cooking facilities and the right side lacks ingredients, the pair must verify their actions are synchronized. The workflow dictates that the left player prepares ingredients and plates, which allows the right player to complete the cooking and serving.
  \item \textbf{Counter Circuit}: This larger ring-shaped map places pots, ingredients, and serving stations on four different sides. Narrow paths make blocking frequent and teamwork difficult. A key strategy for success is placing onions on the central counters to allow for quick hand-offs between players.
\end{itemize}


\begin{figure*}
    \centering
    \includegraphics[width=1.0\textwidth]{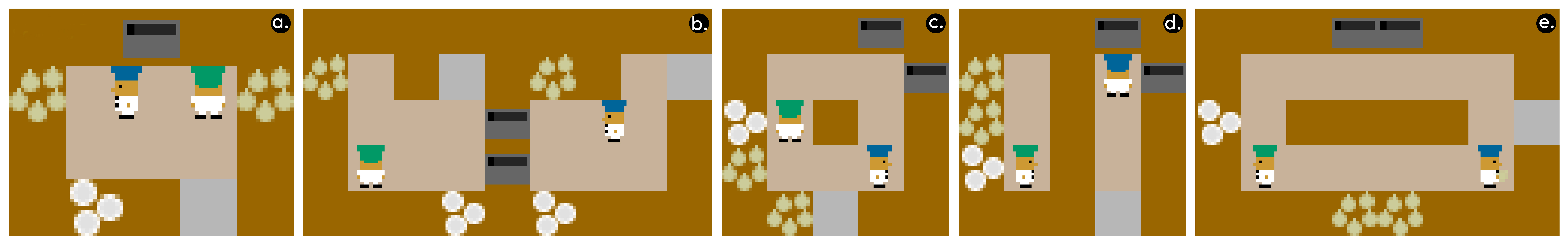}
    \caption{Screenshots of five tasks in the Overcooked environment: (a) Cramped Room, (b) Asymmetric Advantages, (c) Coordination Ring, (d) Forced Coordination, and (e) Counter Circuit.}
    \label{fig:overcooked}
\end{figure*}

\subsection{Pistonball}

In this physics-based cooperative environment, agents control vertically actuating pistons to propel a ball toward the left boundary, as illustrated in Figure~\ref{fig:positionball}. Developing an optimal policy for this task requires the agents to learn and execute highly coordinated joint behaviors. The environment contains 10 pistons, each of which is controlled by an individual agent. All agents receive a shared global reward at each timestep, consisting of a movement-based term and a fixed time penalty. The movement reward is calculated as the net displacement of the ball towards the left, normalized as a percentage of the total initial distance to the wall. Conversely, rightward movement incurs a negative reward. We evaluate performance using the episode return as the metric.


\begin{figure}
    \centering
    \includegraphics[width=\columnwidth]{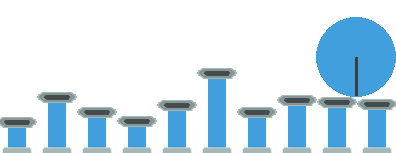}
    \caption{Screenshot of the Pistonball environment.}
    \label{fig:positionball}
\end{figure}

 \section{Experimental Details on LMM-Based Pairwise Comparisons}
\label{appendix:acc}

\begin{figure*}[htbp]
\centering
\begin{subfigure}{0.32\linewidth}
  \centering
  \includegraphics[width=\linewidth]{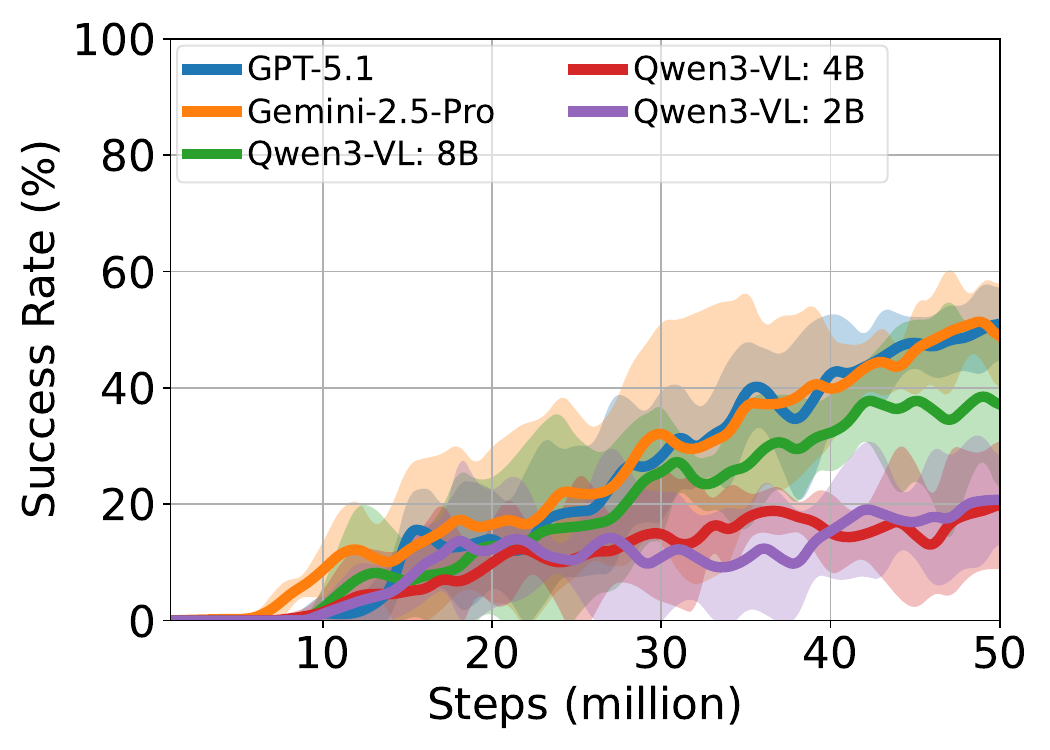}
  \caption{Pass Gate 1 query}
  \label{fig:pg_count1}
\end{subfigure}
\hfill
\begin{subfigure}{0.32\linewidth}
  \centering
  \includegraphics[width=\linewidth]{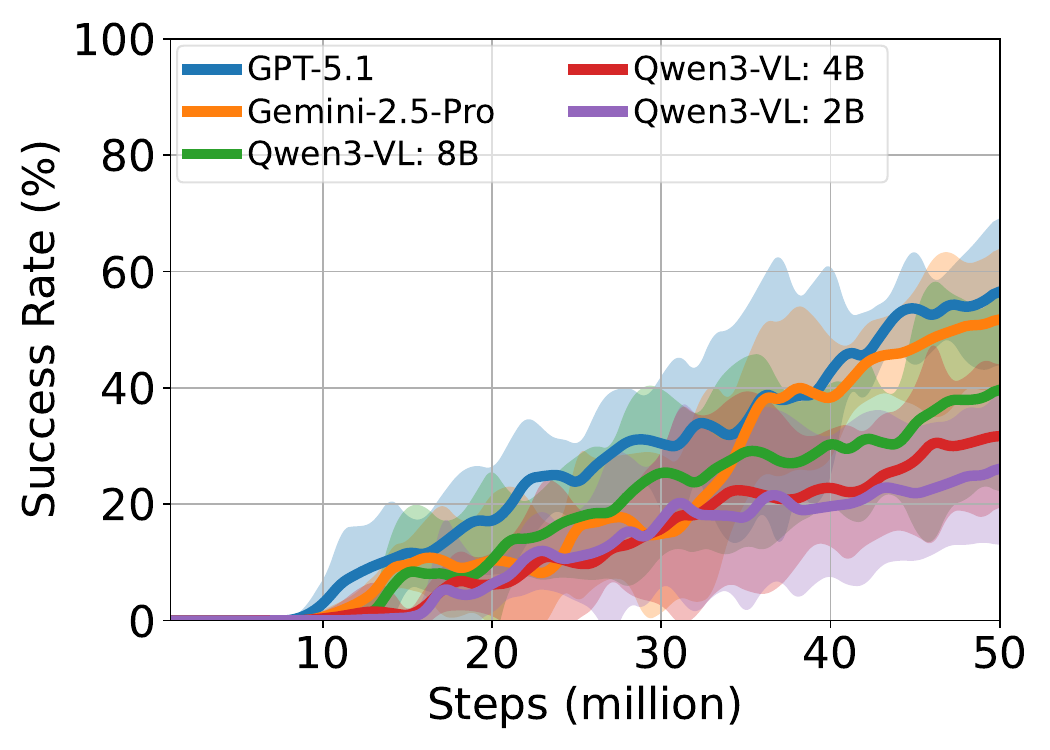}
  \caption{Pass Gate 4 queries}
  \label{fig:pg_count4}
\end{subfigure}
\hfill
\begin{subfigure}{0.32\linewidth}
  \centering
  \includegraphics[width=\linewidth]{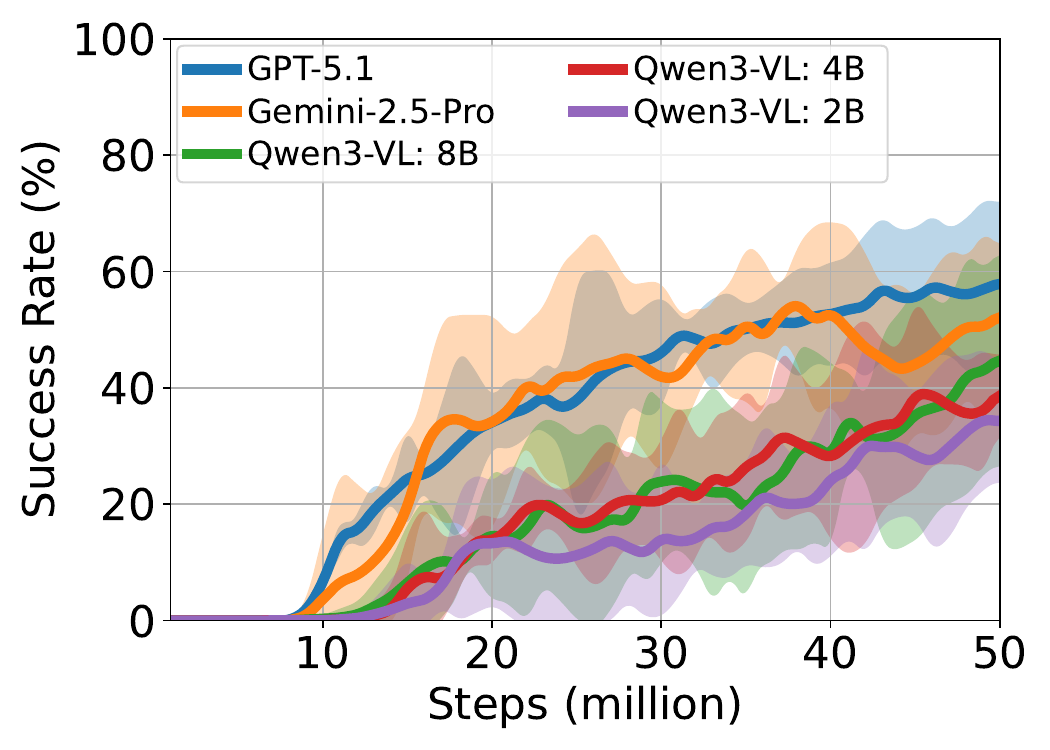}
  \caption{Pass Gate 8 queries}
  \label{fig:pg_count8}
\end{subfigure}
\hfill
\begin{subfigure}{0.32\linewidth}
  \centering
  \includegraphics[width=\linewidth]{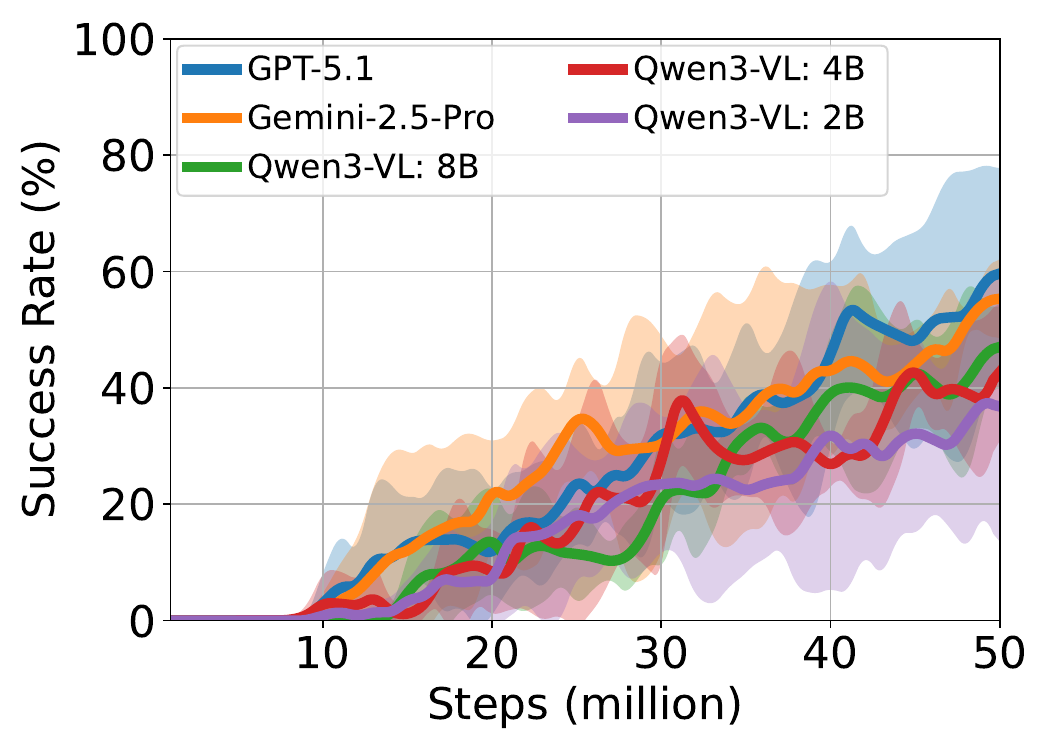}
  \caption{Pass Gate 16 queries}
  \label{fig:pg_count16}
\end{subfigure}
\hfill
\begin{subfigure}{0.32\linewidth}
  \centering
  \includegraphics[width=\linewidth]{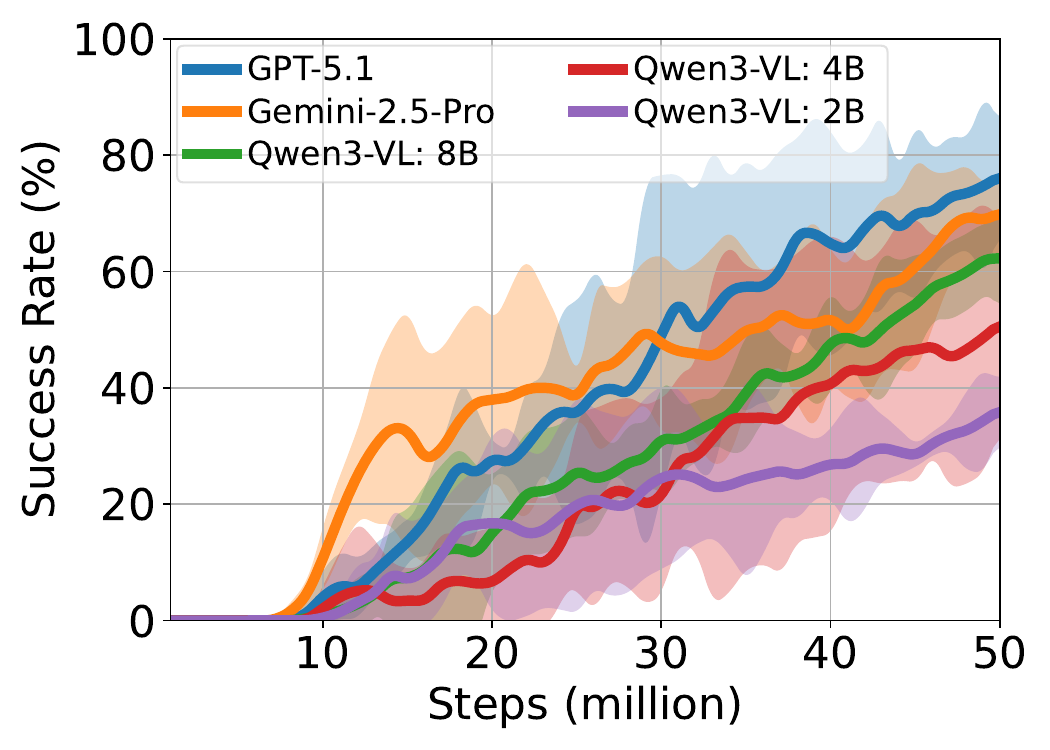}
  \caption{Herd Sheep 1 query}
  \label{fig:hs_count1}
\end{subfigure}
\hfill
\begin{subfigure}{0.32\linewidth}
  \centering
  \includegraphics[width=\linewidth]{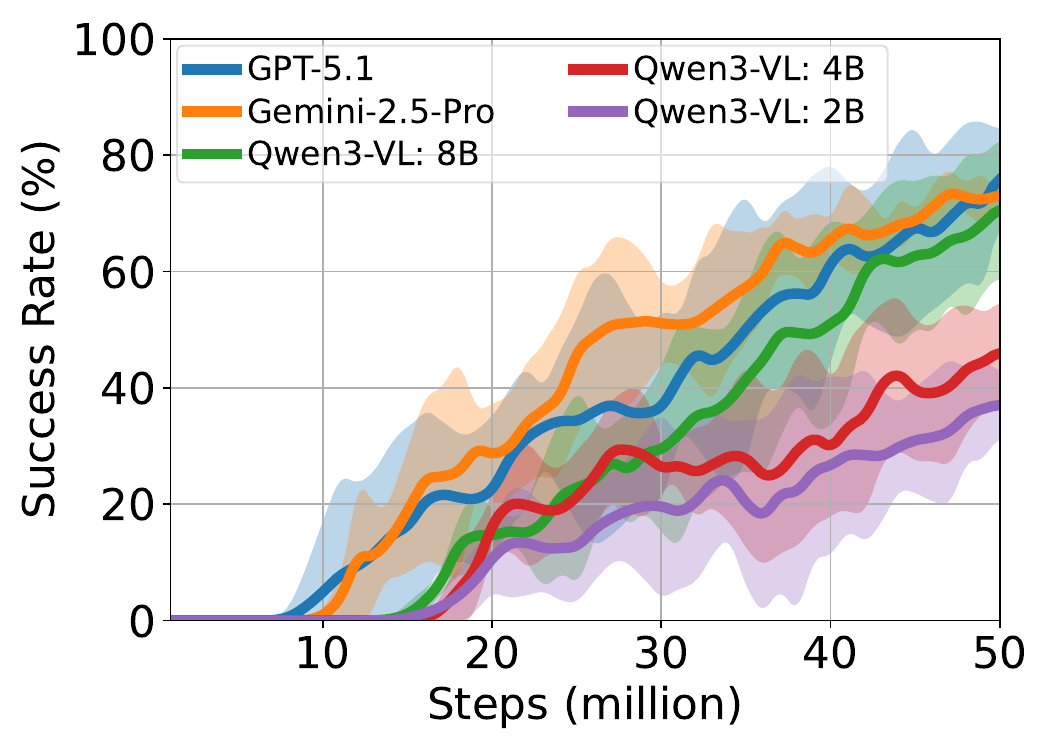}
  \caption{Herd Sheep 4 queries}
  \label{fig:hs_count4}
\end{subfigure}
\hfill
\begin{subfigure}{0.32\linewidth}
  \centering
  \includegraphics[width=\linewidth]{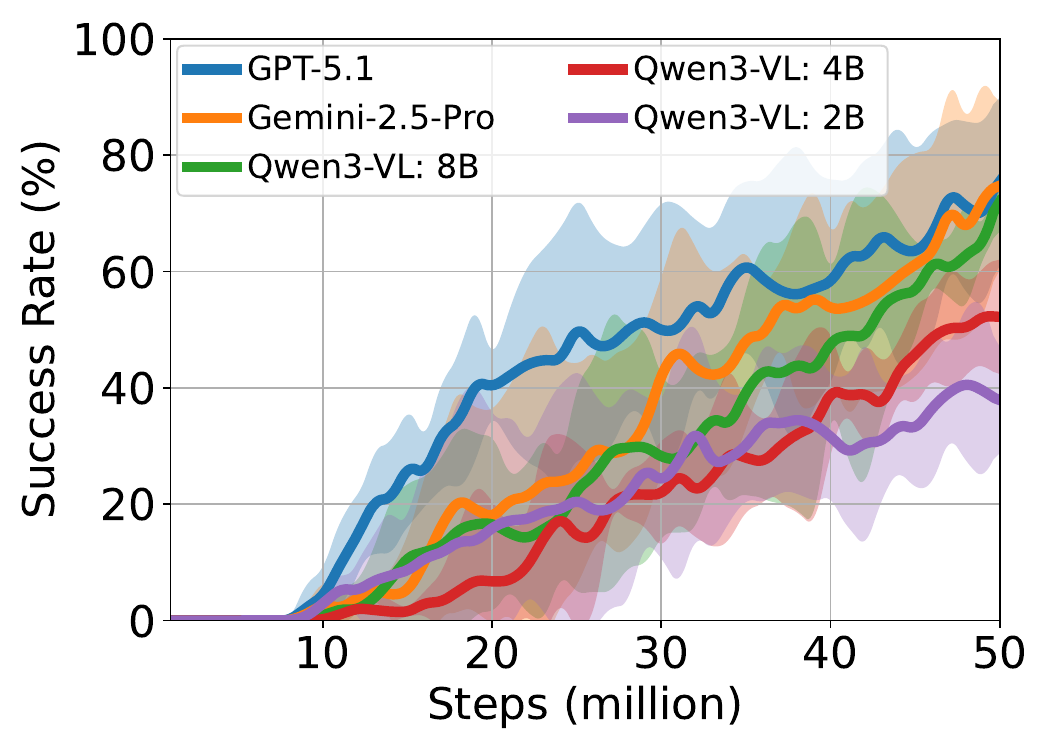}
  \caption{Herd Sheep 8 queries}
  \label{fig:hs_count8}
\end{subfigure}
\hfill
\begin{subfigure}{0.32\linewidth}
  \centering
  \includegraphics[width=\linewidth]{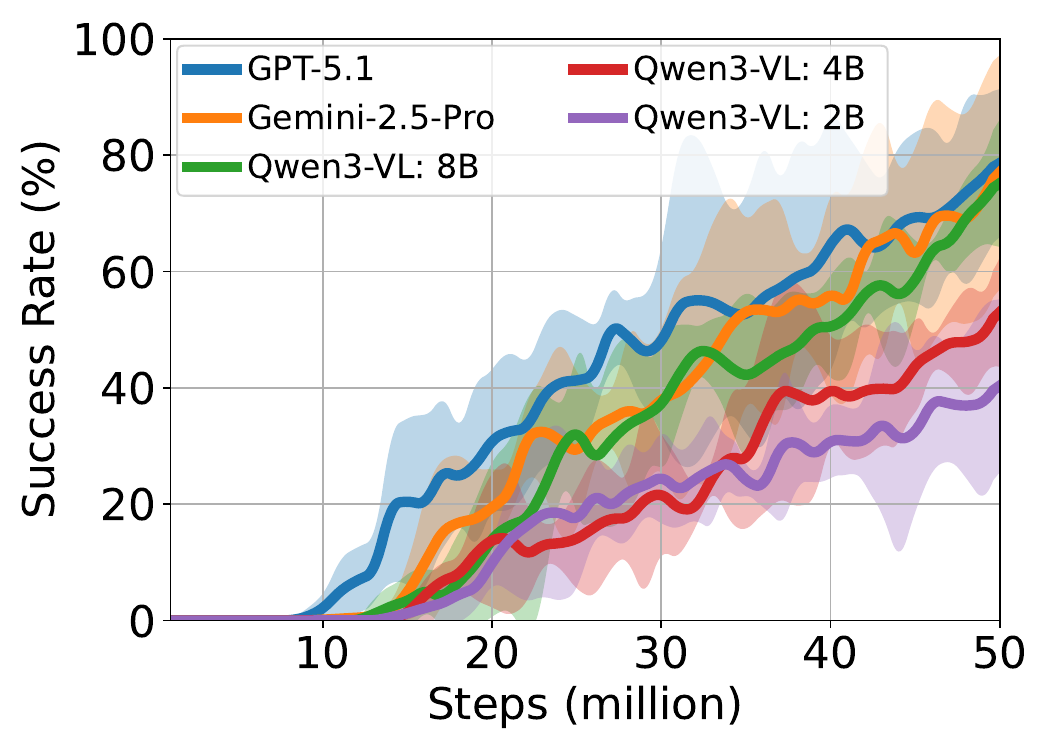}
  \caption{Herd Sheep 16 queries}
  \label{fig:hs_count16}
\end{subfigure}
\hfill
\begin{subfigure}{0.32\linewidth}
  \centering
  \includegraphics[width=\linewidth]{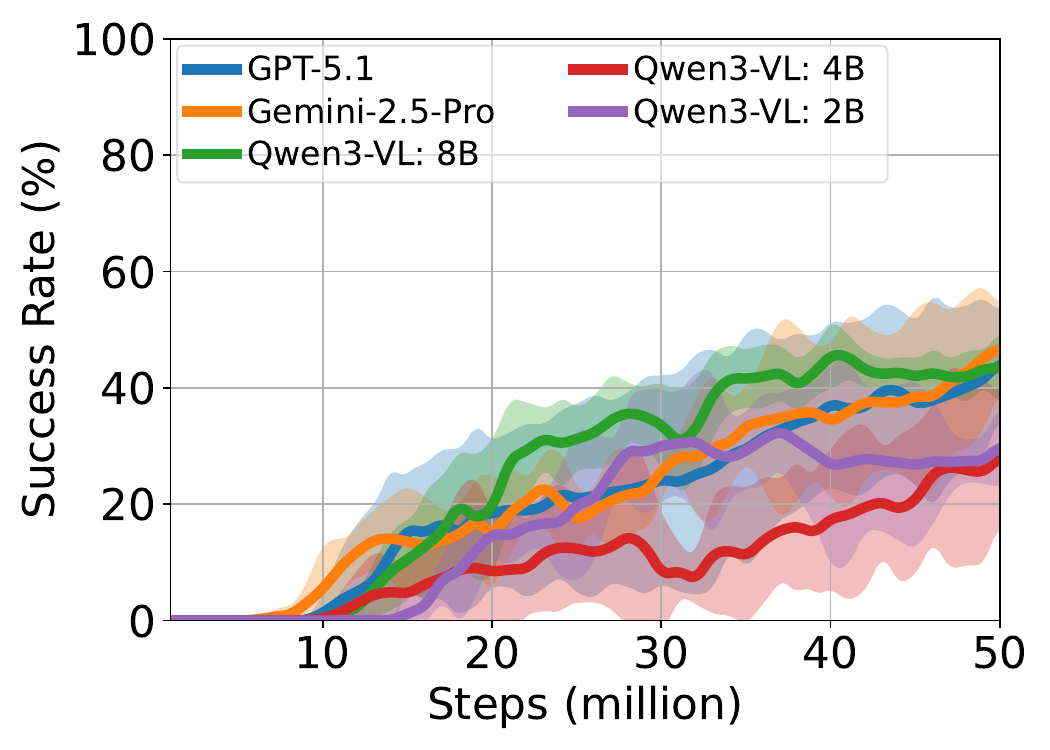}
  \caption{Collect Ball 1 query}
  \label{fig:hs_count1}
\end{subfigure}
\hfill
\begin{subfigure}{0.32\linewidth}
  \centering
  \includegraphics[width=\linewidth]{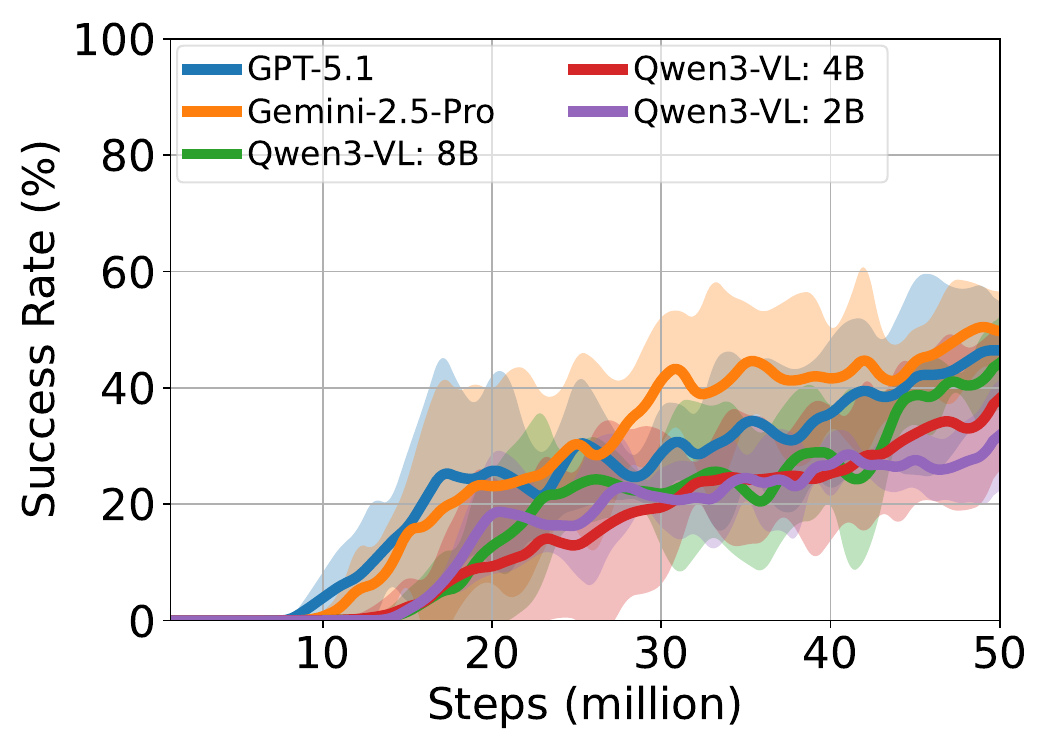}
  \caption{Collect Ball 4 queries}
  \label{fig:hs_count4}
\end{subfigure}
\hfill
\begin{subfigure}{0.32\linewidth}
  \centering
  \includegraphics[width=\linewidth]{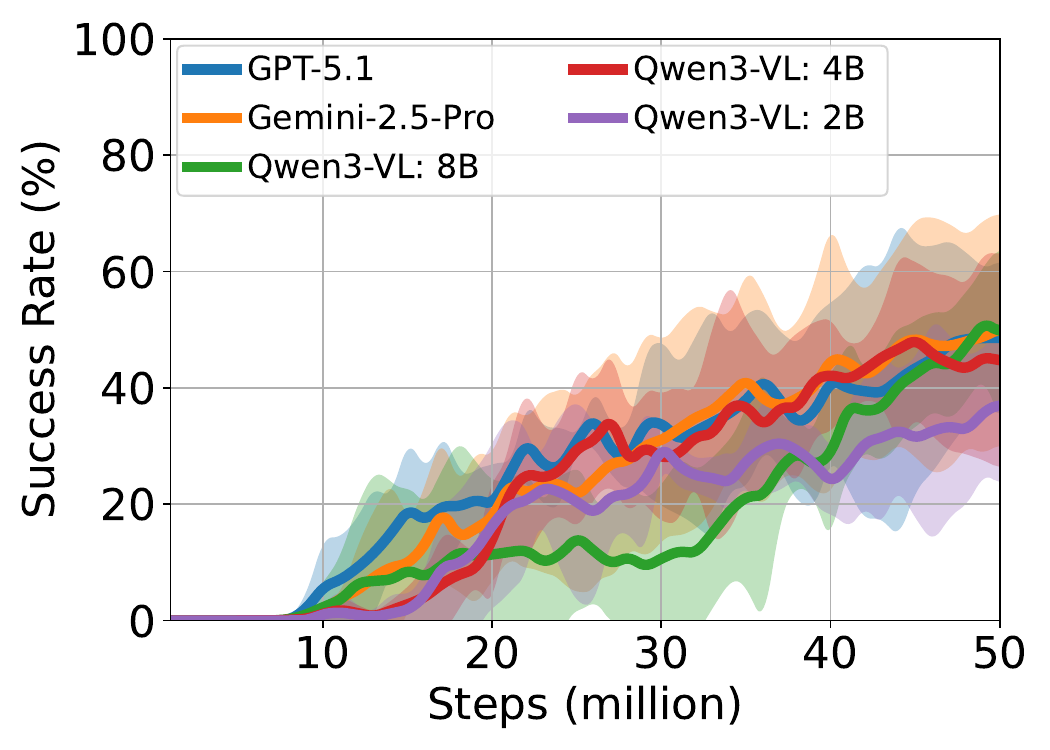}
  \caption{Collect Ball 8 queries}
  \label{fig:hs_count8}
\end{subfigure}
\hfill
\begin{subfigure}{0.32\linewidth}
  \centering
  \includegraphics[width=\linewidth]{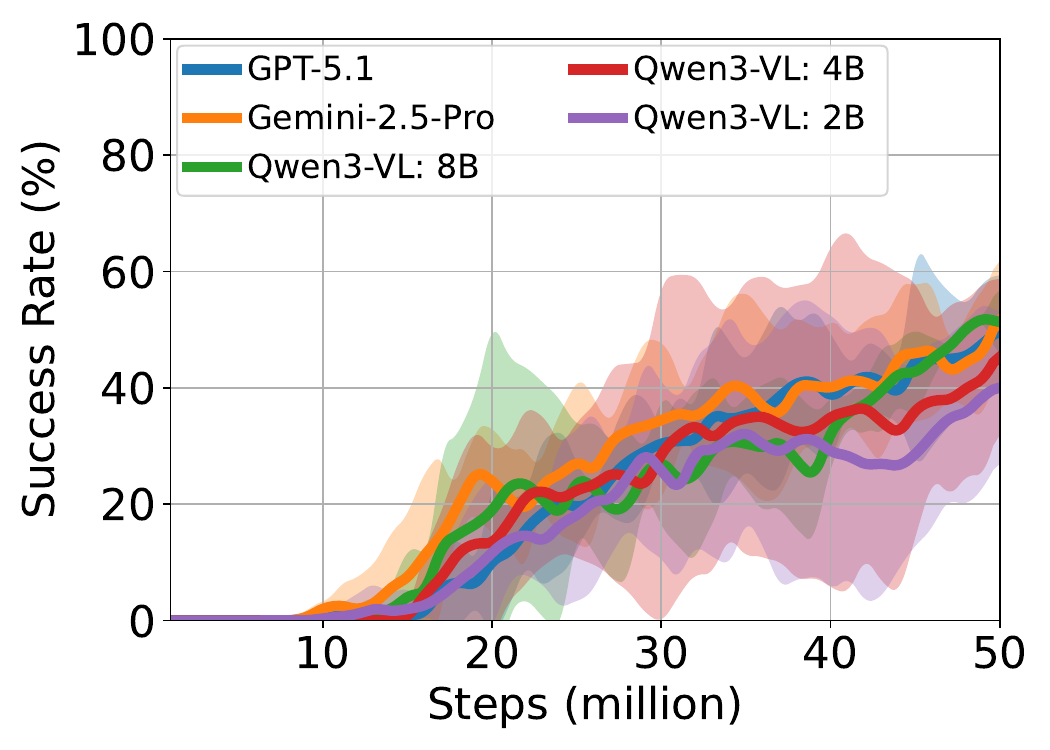}
  \caption{Collect Ball 16 queries}
  \label{fig:hs_count16}
\end{subfigure}
\caption{Training Curves of MARS-RA with Different LMMs and Numbers of Queries for Pairwise Comparisons. Results are averaged over five random seeds, and error bars indicate 95\% confidence intervals.}
\label{fig:different_count}
\end{figure*}


\begin{table*}[t]
\centering

\begin{tabular}{lllll}
\toprule
LMMs & Pass Gate & Herd Sheep & Collect Ball & LMM Average \\
\midrule
GPT-5.1  &     $0.65 \pm0.09$ &    $0.83 \pm0.06$  &      $0.68 \pm0.04$  &  $0.72 \pm0.06$ \\
Gemini-2.5-Pro  &  $0.63 \pm0.04$    &   $0.77 \pm0.10$   &      $0.70 \pm0.03$ &  $0.70 \pm0.04$   \\
Qwen3-VL: 8B  &   $0.48 \pm0.04$   &   $0.62 \pm0.08$   &    $0.46 \pm0.15$  &   $0.52 \pm0.05$   \\
Qwen3-VL: 4B  &   $0.42 \pm0.11$   &  $0.54 \pm0.02$    &     $0.40 \pm0.06$ &   $0.45 \pm0.04$   \\
Qwen3-VL: 2B  &    $0.37 \pm0.18$  &    $0.49 \pm0.13$  &    $0.42 \pm0.10$  &   $0.43 \pm0.03$   \\

\midrule
Task Average   &    $0.51 \pm0.06$  &    $0.65 \pm0.06$  &    $0.53 \pm0.07$  &   $0.56 \pm0.06$   \\
\bottomrule
\end{tabular}
\caption{Pairwise comparison accuracy of different LMMs on the three MARS-Bench tasks, with standard error.}
\label{tab:acc}
\end{table*}

\begin{figure*}[t]
  \centering
  \begin{subfigure}{0.45\linewidth}
    \centering
    \includegraphics[width=\linewidth]{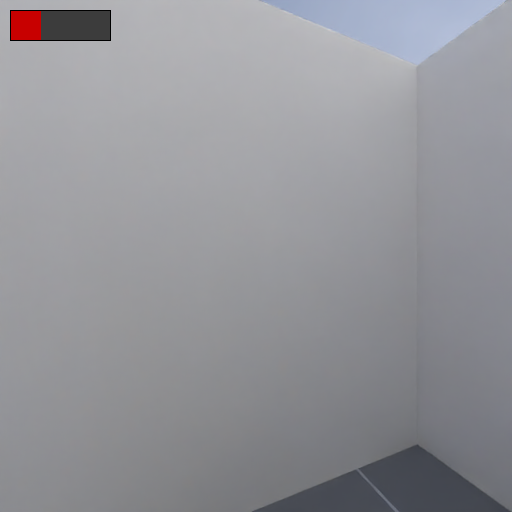}
    \caption{Limited visual observations.}
  \end{subfigure}
  \hfill
  \begin{subfigure}{0.45\linewidth}
    \centering
    \includegraphics[width=\linewidth]{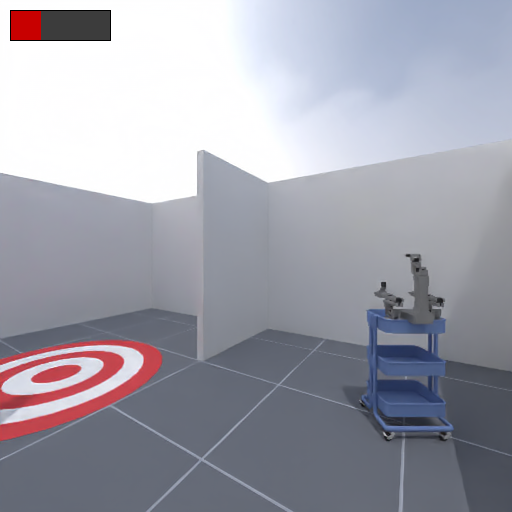}
    \caption{Informative visual observations.}
  \end{subfigure}
  \caption{Pairwise comparison errors by LMMs mainly occur when both agents’ egocentric observations are limited, as shown in the left image. In contrast, correct judgments can be made as long as at least one agent provides an informative observation with rich visual cues, as shown in the right image.}
  \label{fig:two_case}
\end{figure*}

\subsection{Setup} 

We train MARS-RA on the three MARS-Bench tasks using Gemini-2.5-Pro, GPT-5.1, and Qwen3-VL (2B, 4B, and 8B) under varying numbers of pairwise comparison queries, and concurrently measure the accuracy of LMM-based pairwise comparisons. The accuracy of LMM-based pairwise comparisons is computed by measuring the agreement between LMM judgments generated during training and the corresponding per-agent dense reward signals of each task, and then averaging across comparisons. All other experimental settings follow Section~\ref{sec:experiment} and Appendix~\ref{appendix:exp_details}.

\subsection{Results} Figure~\ref{fig:different_count} shows the training curves of different LMMs under varying numbers of pairwise comparison queries. We observe that both employing LMMs with higher pairwise comparison accuracy and increasing the number of pairwise comparison queries lead to improved MARS-RA performance. However, when using high-accuracy LMMs, the performance gains from increasing the query count are less pronounced compared to those achieved with lower-accuracy LMMs. These results suggest a trade-off and complementary relationship between LMM pairwise comparison accuracy and the number of comparison queries.  Table~\ref{tab:acc} reports the accuracy of LMM-based pairwise comparisons. We observe that commercial models generally outperform open-source models, and models with larger parameter scales tend to achieve higher accuracy. However, the accuracy of all models still leaves room for improvement. Figure~\ref{fig:two_case} illustrates typical cases observed during LMM-based pairwise comparisons. When an agent’s egocentric observation is severely limited and lacks sufficient visual information (e.g., teammate positions or goal locations), LMMs can still make correct judgments as long as at least one agent’s observation contains informative visual cues. In contrast, pairwise comparison errors most frequently occur when all agents face a wall and lack informative visual reference cues.

\section{Real-World Validation Details}
\label{appendix:realworld}

\begin{figure}
    \centering
    \includegraphics[width=\columnwidth]{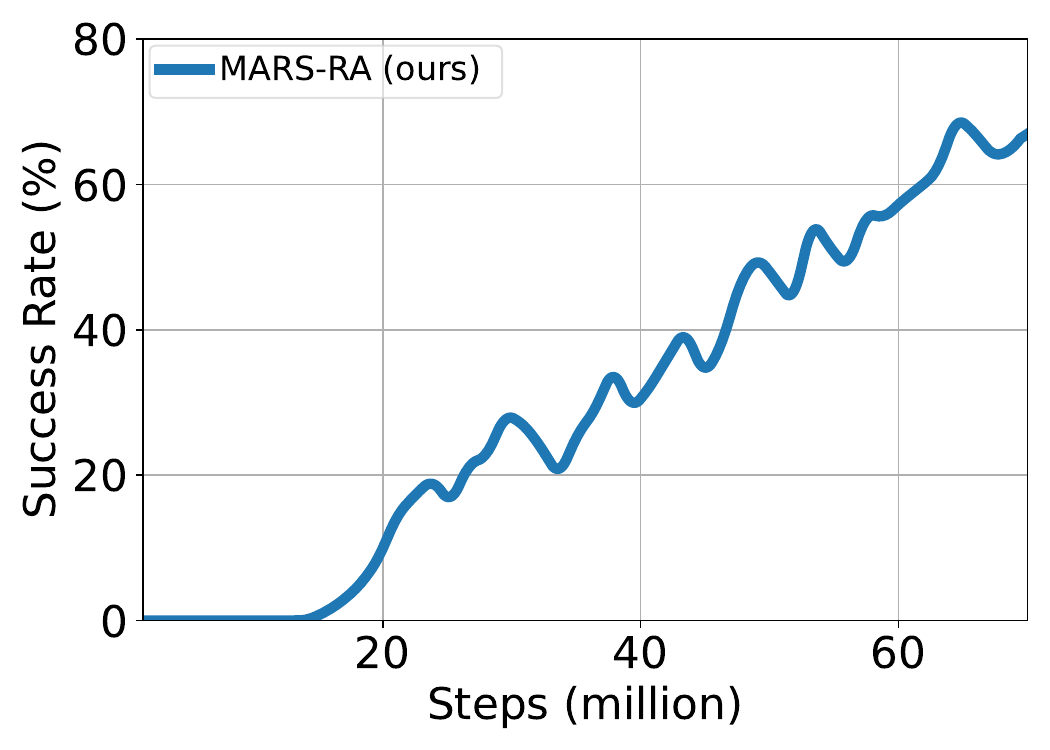}
    \caption{The learning curve obtained during training in the simulated environment for the real-world validation.}
    \label{fig:real_line}
\end{figure}

\begin{figure*}[htbp]
\centering
\begin{subfigure}{0.32\linewidth}
  \centering
  \includegraphics[width=\linewidth]{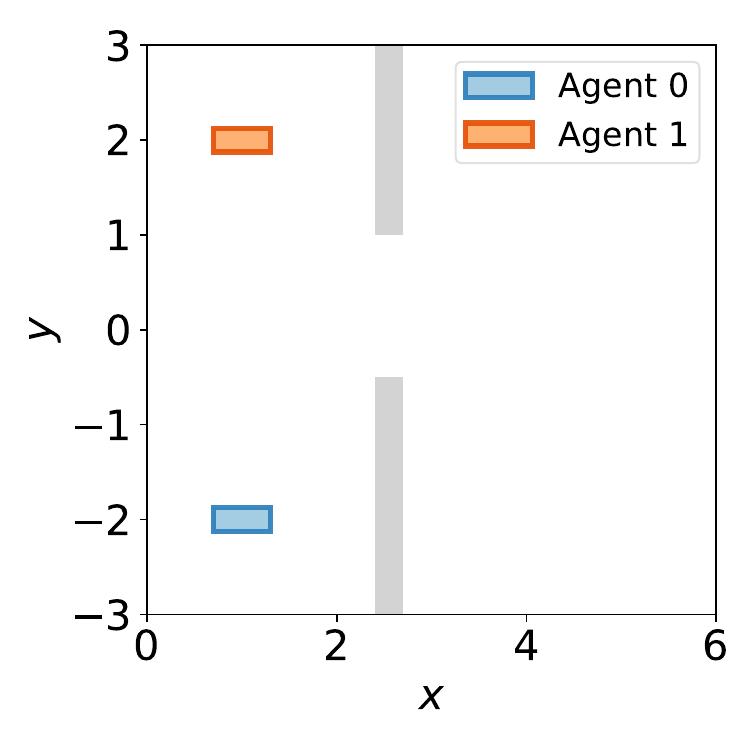}
  \caption{Agent positions at t = 0.0 s.}
  \label{fig:no1}
\end{subfigure}
\hfill
\begin{subfigure}{0.32\linewidth}
  \centering
  \includegraphics[width=\linewidth]{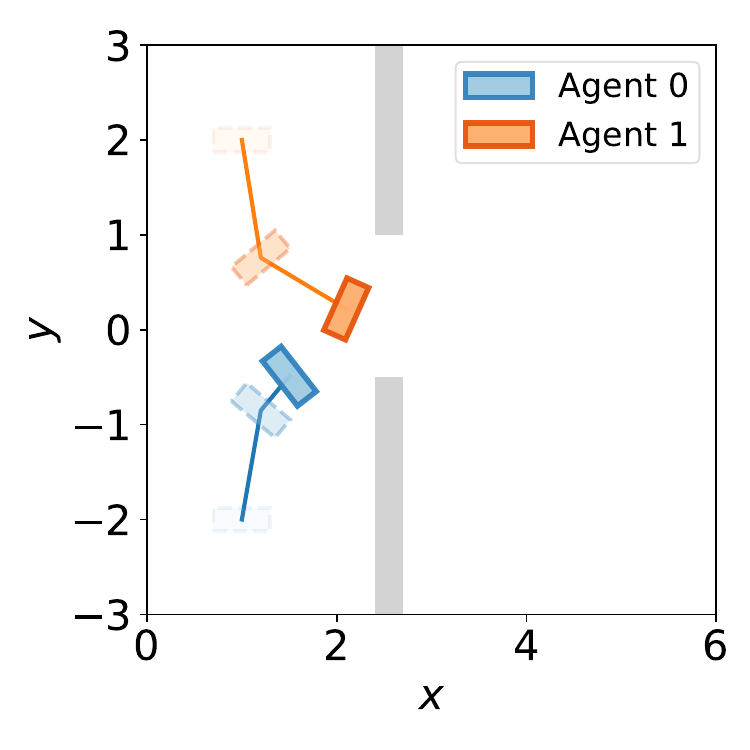}
  \caption{Agent positions at t = 16.6 s.}
  \label{fig:no2}
\end{subfigure}
\hfill
\begin{subfigure}{0.32\linewidth}
  \centering
  \includegraphics[width=\linewidth]{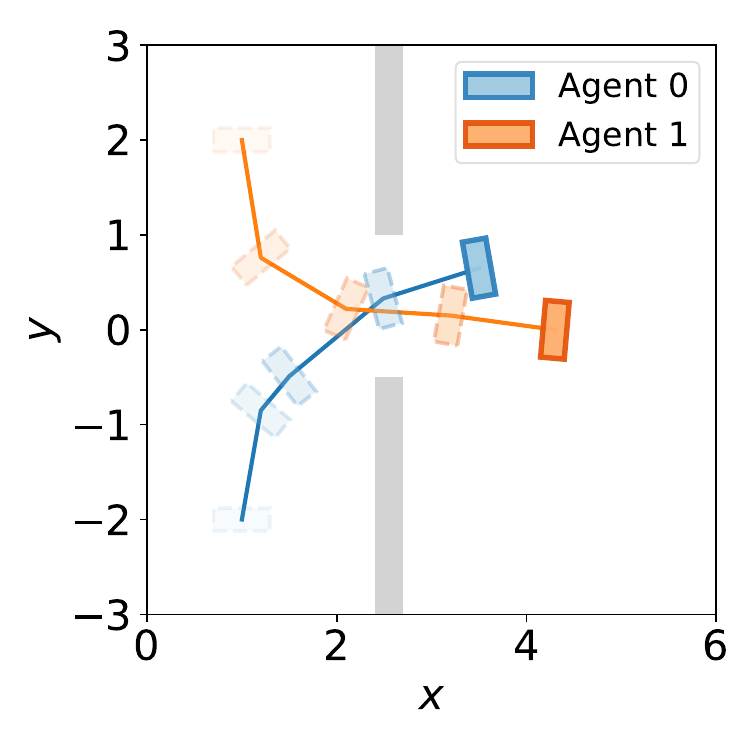}
  \caption{Agent positions at t = 31.7 s.}
  \label{fig:no3}
\end{subfigure}
\caption{A representative successful real-world trajectory of XLeRobots in the Pass Gate task, where a policy trained with MARS-RA enables coordinated, sequential gate traversal without collisions. Blue and yellow rectangles represent Agent~0 and 1, respectively; opaque rectangles indicate current positions, and transparent rectangles denote historical positions.}
\label{fig:trajectory}
\end{figure*}

\subsection{Task Setup}
We conduct a real-world validation of MARS-RA, as illustrated in Figure~\ref{fig:real_world}. We deploy two XLeRobot robots to perform the Pass Gate task in a real-world indoor environment. The entire room is first captured via 3D scanning and reconstructed as a virtual 3D environment, which is then instantiated as a task in MARS-Bench. MARS-RA is trained in simulation for 70 million environment steps, using Qwen3-VL (8B) to generate pairwise comparisons, with 16 comparisons per decision. We additionally measure the pairwise comparison accuracy of Qwen3-VL (8B) on the Pass Gate task during training. The success rate is used as the evaluation metric for this experiment. The success criterion in simulation is consistent with that described in Section~\ref{sec:experiment}. In the real-world setting, success is defined as both agents passing through the gate into the adjacent room within two minutes without collisions; otherwise, the trial is considered a failure. All other settings follow Section~\ref{sec:experiment}, Appendix~\ref{appendix:exp_details} and Appendix~\ref{appendix:acc}.

\subsection{Results}
The training results in simulation are shown in Figure~\ref{fig:real_line}. On the Pass Gate task, the success rate exceeds 67\% after training. The pairwise comparison accuracy of Qwen3-VL (8B) is 71\%. This result preliminarily suggests the potential of MARS-RA to guide agents toward effective cooperative policies in scenarios with real-world–level complexity. 

We deploy the trained policy in the real world and conduct 25 trials, of which 16 are successful, resulting in a success rate of 64\%. Notably, LMM-based pairwise comparisons are only used during training in MARS-RA, and are not required at deployment time. Figure~\ref{fig:trajectory} shows a representative successful trajectory of the trained policy in a real-world environment. Both agents initially move toward the gate simultaneously. The yellow agent passes through the gate first, while the blue agent maintains a safe distance and waits. After the yellow agent clears the gate and moves forward to create sufficient space, it comes to a stop, allowing the blue agent to pass through the gate once a safe clearance is available.

\end{document}